\documentclass[letterpaper]{article} 
\usepackage{aaai2026}  
\usepackage{times}  
\usepackage{helvet}  
\usepackage{courier}  
\usepackage[hyphens]{url}  
\usepackage{graphicx} 
\usepackage{natbib}  
\usepackage{caption} 
\usepackage{algorithm}
\usepackage{algorithmic}
\usepackage{amsmath,amsfonts, amssymb}
\usepackage{array}
\usepackage{subcaption}

\usepackage{textcomp}
\usepackage{url}
\usepackage{verbatim}
\usepackage{graphicx}
\usepackage{tabularx}
\usepackage{multirow}
\usepackage{cite}
\usepackage{booktabs}

\usepackage{tikz}
\usepackage{amsmath}
\usetikzlibrary{shapes.geometric, arrows.meta, positioning, fit, calc, backgrounds, shadows}
\usepackage[T1]{fontenc}

\newtheorem{definition}{Definition}
\newtheorem{challenge}{Challenge}
\newtheorem{theorem}{Theorem}

\newtheorem{proof}{Proof}
\newtheorem{implications}{Implication}

\usepackage{tikz}
\newcommand*\filledcircled[2][\normalsize]{%
  \tikz[baseline=(char.base)]{
    \node[shape=circle,fill,inner sep=0.5pt] (char) {#1\textcolor{white}{#2}};}}
\newcommand{\mymethod}{\textbf{\texttt{CAMEL}}}

\usepackage{newfloat}
\usepackage{listings}

\DeclareCaptionStyle{ruled}{labelfont=normalfont,labelsep=colon,strut=off} 
\floatstyle{ruled}
\newfloat{listing}{tb}{lst}{}
\floatname{listing}{Listing}
\title{Drift-aware Collaborative Assistance Mixture of Experts for Heterogeneous Multistream Learning}
\author{En Yu, Jie Lu, Kun Wang, Xiaoyu Yang, Guangquan Zhang}
\affiliations{
    Australian Artificial Intelligence Institute (AAII)\\
    University of Technology Sydney (UTS), Australia 
}

\usepackage{bibentry}

\begin{document}

\maketitle

\begin{abstract}
Learning from multiple data streams in real-world scenarios is fundamentally challenging due to intrinsic heterogeneity and unpredictable concept drifts. Existing methods typically assume homogeneous streams and employ static architectures with indiscriminate knowledge fusion, limiting generalizability in complex dynamic environments. To tackle this gap, we propose~\mymethod, a dynamic \textbf{C}ollaborative \textbf{A}ssistance \textbf{M}ixture of \textbf{E}xperts \textbf{L}earning framework. It addresses heterogeneity by assigning each stream an independent system with a dedicated feature extractor and task-specific head. Meanwhile, a dynamic pool of specialized private experts captures stream-specific idiosyncratic patterns. Crucially, collaboration across these heterogeneous streams is enabled by a dedicated assistance expert. This expert employs a multi-head attention mechanism to distill and integrate relevant context autonomously from all other concurrent streams. It facilitates targeted knowledge transfer while inherently mitigating negative transfer from irrelevant sources. Furthermore, we propose an Autonomous Expert Tuner (AET) strategy, which dynamically manages expert lifecycles in response to drift. It instantiates new experts for emerging concepts (freezing prior ones to prevent catastrophic forgetting) and prunes obsolete ones. This expert-level plasticity provides a robust and efficient mechanism for online model capacity adaptation. Extensive experiments demonstrate \mymethod’s superior generalizability across diverse multistreams and exceptional resilience against complex concept drifts.
\end{abstract}


\section{Introduction}
\label{sec:intro}
Learning from streaming data has become fundamental to modern intelligent systems, enabling real-time decision-making in dynamic and continuously evolving environments~\cite{cacciarelli2024active,marcu2024big,agrahari2022concept}. A central challenge in streaming learning is concept drift—the phenomenon where the underlying data distribution changes over time—requiring models to continuously adapt in order to maintain predictive performance~\cite{lu2018learning}. While most streaming learning studies focus on single-stream settings~\cite{jiao2022incremental,wen2023onenet}, many real-world applications inherently involve multiple concurrent data streams. For example, a smart city platform integrates traffic sensor feeds, weather reports, public transportation logs, and social media sentiment streams. These streams evolve independently yet often carry latent correlations that, if exploited effectively, can provide complementary information for more accurate and robust decision-making~\cite{xiang2023concept,read2025supervised,ma2024multiview}. Capturing such dynamic inter-stream relationships while adapting to concept drift is crucial for advancing streaming learning toward practical deployment~\cite{yang2025adapting,xu2025coral}.

Despite recent progress, existing multistream learning methods are caught in a critical dilemma. On the one hand, most approaches operate under a homogeneous space assumption, which presumes that all streams share the same feature and label spaces~\cite{yu2024online, jiao2023reduced}. This assumption fails to deal with the intrinsic heterogeneity commonly present in practical applications, where streams may originate from distinct feature spaces or predictive objectives due to different data sources~\cite{korycki2021concept,panchal2023flash}. On the other hand, prevailing methods typically employ a monolithic and static architecture, either retrained or incrementally fine-tuned~\cite{xu2025drift2matrix,wang2021evolving}. This design suffers from critical limitations in multistream environments, e.g., retraining induces catastrophic forgetting of prior knowledge, while fine-tuning becomes fragile under asynchronous drifts, where adapting to one stream’s evolution can degrade performance on others. The lack of structural flexibility and targeted adaptation thus prevents robust performance across heterogeneous evolving streams. 

To bridge this gap, we formalize the problem as Heterogeneous Multistream Learning (HML), where multiple concurrent data streams exhibit intrinsic heterogeneity, latent inter-stream correlations, and asynchronous concept drifts. Specifically, \textit{\textbf{1) Intrinsic Heterogeneity}}: feature and label spaces across streams differ in dimensionality and semantics, precluding direct application of homogeneous models; \textit{\textbf{2) Knowledge Fusion}}: while streams may contain useful correlations, such relationships are dynamic and selective, requiring mechanisms that can leverage relevant information while avoiding negative transfer from irrelevant streams; and \textit{\textbf{3) Asynchronous Concept Drifts}}: streams evolve independently with diverse drift patterns, demanding flexible and stream-specific adaptation. These challenges necessitate a generalized and drift-aware learning framework that can handle stream-wise specialization while enabling intelligent knowledge fusion across heterogeneous drifting streams.

To address these challenges, we propose \mymethod, a dynamic Collaborative Assistance Mixture of Experts Learning framework tailored for heterogeneous data streams. It introduces a modular drift-aware architecture that explicitly addresses the three core challenges. First, to handle intrinsic heterogeneity, we assign each stream a specific learning system comprising a dedicated feature extractor, a private expert pool, and a task-specific prediction head, ensuring stream-specific specialization. Second, to enable adaptive and selective knowledge fusion, \mymethod~incorporates a novel collaborative assistance mechanism. It employs a dedicated attention-based expert per stream dynamically distills relevant contextual information from all other concurrent streams on demand, effectively capturing latent inter-stream correlations while inherently mitigating negative transfer~\cite{vaswani2017attention}. Third, to cope with asynchronous concept drifts, an Autonomous Expert Tuner (AET) is proposed, which monitors drift signals by a distribution-based drift detector and performance indicators per stream, dynamically adding new experts for emerging concepts and pruning obsolete ones. This expert-level plasticity allows our method to autonomously restructure its capacity and specialization over time. Extensive experiments on diverse synthetic and real-world multistream scenarios demonstrate the superior adaptability and generalization ability of our method compared to existing state-of-the-art methods. In summary, our main contributions are:
\begin{itemize}
    \item We propose \mymethod, a generalized and dynamic MoE framework that learns from multiple data streams characterized by heterogeneous features, diverse label spaces and asynchronous concept drifts.
    \item We introduce a collaborative assistance mechanism, where dedicated attention-based experts perform targeted knowledge fusion, providing a effective and adaptive solution to the challenge of positive knowledge transfer.
    \item We design an autonomous tuning strategy that manages the expert lifecycle at a modular level (adding/pruning experts), offering a more robust and interpretable way for drift adaptation.
    \item Comprehensive experiments and theoretical analysis validate the generalizability and robustness of our method across complex synthetic and real-world HML scenarios.
\end{itemize}

\begin{figure*}[thbp]
    \centering
    \includegraphics[width=\textwidth]{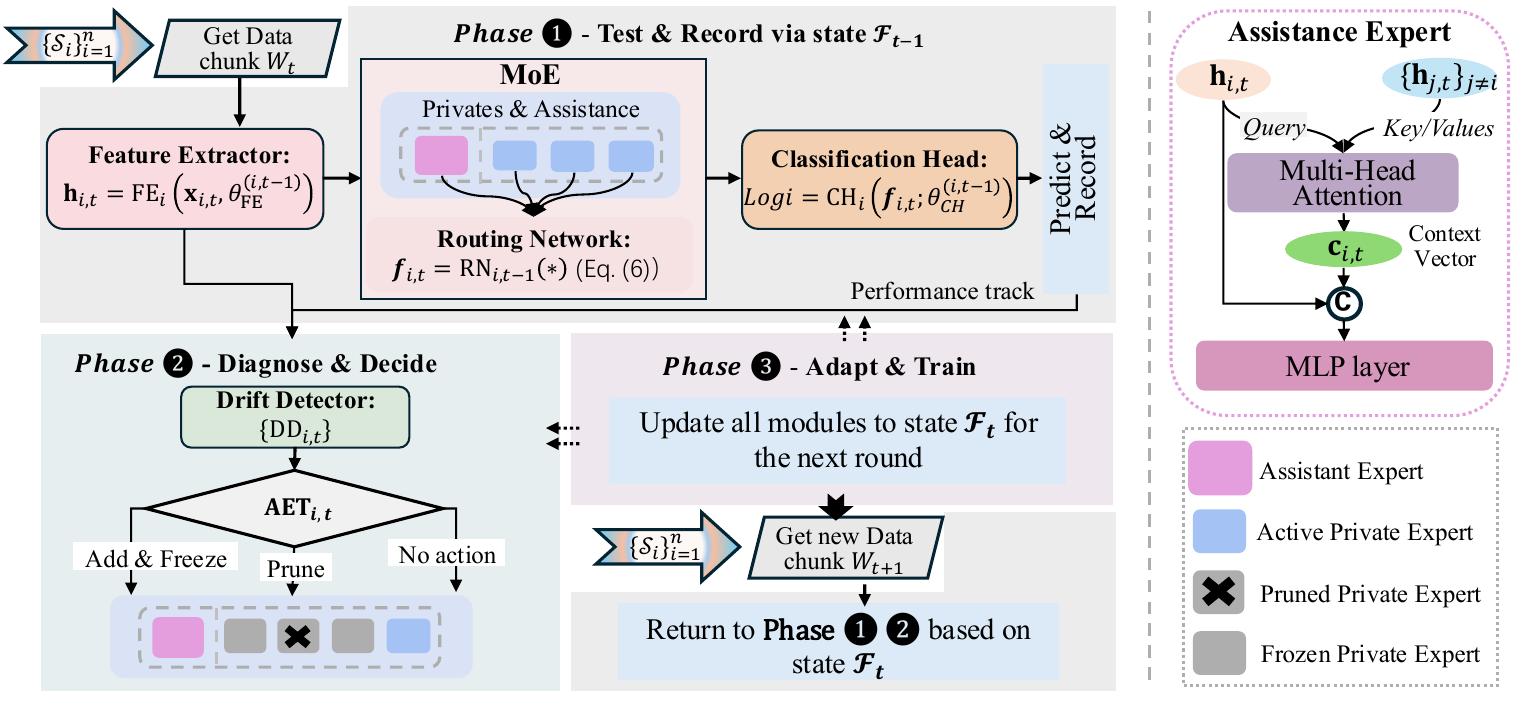}
    \caption{The overall framework of \mymethod. Concretely, each stream's MoE module leverages a dynamic pool of private experts and a dedicated \textit{assistance expert} that performs collaborative fusion via multi-head attention. The entire system follows a Test-Diagnose-Adapt cycle where an Autonomous Expert Tuner (AET) dynamically manages the expert lifecycle (adding/freezing/pruning) in response to drift and performance signals, ensuring continuous adaptation in the HML environment.}
    \label{fig:architecture}
\end{figure*}

\section{Related Works}
\label{sec:relatedworks}
\noindent
\textbf{Stream Learning.}
Early research in streaming learning primarily addressed single-stream scenarios with concept drift~\cite{wan2024online,li2022ddg,kim2024quilt,yu2025learning}, broadly falling into two paradigms: \emph{1) informed methods} integrate explicit drift detection mechanisms to trigger model adaptation based on distribution variations or error signals~\cite{bifet2007learning,gomes2019streaming,lu2025early}, while \emph{2) adaptive approaches} employ detection-free strategies that continuously adjust model parameters in response to evolving data dynamics~\cite{guo2021selective,brzezinski2013reacting,jiao2022incremental}. Recognizing the ubiquity of concurrent streams, recent multistream learning works can be summarized into two categories: \emph{1) Multistream classification} aims to transfer knowledge from labeled source streams to unlabeled targets, such as MCMO using multi-objective feature selection~\cite{jiao2023reduced}, OBAL dynamically weighting streams via drift-aware boosting~\cite{yu2024online}, and BFSRL learning fuzzy shared representations across streams~\cite{yu2024fuzzy}; \emph{2) Multistream collaborative prediction} exploits complementary information across streams for joint forecasting, typically adopting test-then-adapt schemes.  For instance, Wang et al.~\cite{wang2024adaptive} propose adaptive stacking that selectively retrains models for knowledge fusion during drift adaptation, while Wen et al.~\cite{wen2023onenet} employ dual-branch networks separately modeling temporal and cross-variable dependencies. Similarly, CORAL~\cite{xu2025coral} leverages the kernel-induced self-representation method for co-evolving time series. However, both paradigms predominantly assume homogeneous feature spaces and shared label semantics, fundamentally struggling with heterogeneous heterogeneity and asynchronous drifts. 

\noindent
\textbf{Mixture-of-Experts (MoE).} The MoE paradigm achieves scalable, efficient modeling through conditional computation, where a routing mechanism dynamically activates specialized sub-networks ("experts")~\cite{mu2025comprehensive,lei2024adapted}. This architecture demonstrates strong capabilities in multi-task coordination and continual learning~\cite{qin2020multitask,litheory,lei2024adapted} with its sparse activation property preserving computational efficiency while maintaining high model capacity~\cite{sarkar2023edge,tran2025revisiting}. These inherent advantages naturally align with streaming learning's core challenges, including complex pattern recognition, concept drift adaptation and computational constraints. However, MoE frameworks remain largely unexplored for streaming scenarios while exhibiting critical limitations in HML: expert specialization is statically predefined for coarse task categories without mechanisms to dynamically reconfigure expertise for dynamic scenarios, while routing strategies optimize isolated objectives while neglecting knowledge transfer between complementary experts. Our approach fundamentally advances this paradigm through a correlation-aware expert synthesis framework that jointly models latent task dependencies and expert synergies, enabling real-time expert reorganization and coordinated optimization of both routing precision and cross-expert knowledge transfer, unlocking adaptive capacity allocation for evolving data streams.

\section{Preliminary}
\label{sec:preliminary}



\begin{definition}[Heterogeneous Multistream Learning]
\label{def:hybrid_multistream}
Let $\mathcal{S}=\{\mathcal{S}_i\}_{i=1}^n$ be a set of $n$ concurrent data streams. Each stream $\mathcal{S}_i$ is an ordered sequence of instances $\{(\boldsymbol{x}_{i,t}, y_{i,t})\}_{t=1}^\infty$, where $\boldsymbol{x}_{i,t} \in \mathcal{X}_i \subseteq \mathbb{R}^{D_i}$ is the feature vector from a stream-specific feature space of dimensionality $D_i$, and $y_{i,t} \in \mathcal{Y}_i = \{1, \dots, C_i\}$ is the corresponding class label from a stream-specific label space of size $C_i$. 
The underlying joint distribution $P_t^{(i)}(\boldsymbol{x}_{i,t},y_{i, t})$ for each stream $\mathcal{S}_i$ can change over time, exhibiting concept drift. The goal in HML is to design an adaptive mechanism $\mathcal{F}: \{\mathcal{X}_i \to \mathcal{Y}_i\}_{i=1}^n$ that continuously adapts to predict new data from each stream.
\end{definition}
As mentioned before, three main challenges must be addressed simultaneously in HML, i.e., \textit{Intrinsic Heterogeneity}, \textit{Knowledge Fusion} and \textit{Asynchronous Drifts}. These challenges are defined as follows,
\begin{challenge}[Intrinsic Heterogeneity]
\label{cha:heterogeneity}
Real-world multistream scenarios exhibit intrinsic heterogeneity in both feature and label spaces across streams. For any pair of streams \(\mathcal{S}_i\) and \(\mathcal{S}_j\) (\(i \neq j\)), their respective feature spaces may differ in dimensionality (\(D_i \neq D_j\)) and attribute structure (\(\mathcal{X}_i \neq \mathcal{X}_j\)), while their label spaces can define disjoint predictive tasks (\(\mathcal{Y}_i \neq \mathcal{Y}_j\) implying \(C_i \neq C_j\)). 
\end{challenge}

    

\begin{challenge}[Knowledge Fusion]
\label{cha:fusion}
While the streams $\mathcal{S}$ are heterogeneous, they may contain latent time-varying correlations that can be exploited for mutual benefit. The core challenge is to design a mechanism for selective and adaptive knowledge fusion. This requires simultaneously achieving two conflicting objectives for any given stream $\mathcal{S}_i$. First, the model must be able to identify and leverage useful contextual information from all other concurrent streams $\{\mathcal{S}_j\}_{j \neq i}$ to enhance its predictive capability for $\mathcal{S}_i$. Second, it must be robust to dynamically ignore information from any stream $\mathcal{S}_j$ that is irrelevant or contains misleading patterns, thereby avoiding negative transfer. 
\end{challenge}

\begin{challenge}[Asynchronous Drifts]
\label{cha:asyn}
The non-stationarity of each stream $\mathcal{S}_i$ presents that its data-generating distribution $P_t^{(i)}$ evolves with uncoordinated and diverse dynamics. These concept drifts are both \textit{asynchronous} and \textit{diverse}. Formally, for any two streams $\mathcal{S}_i$ and $\mathcal{S}_j$ (where $i \neq j$), $\exists t$, $P_t^{(i)}(y|\boldsymbol{x}) \neq P_{t+1}^{(i)}(y|\boldsymbol{x})$ while $P_t^{(j)}(y|\boldsymbol{x}) = P_{t+1}^{(j)}(y|\boldsymbol{x})$. Furthermore, the drift patterns vary across streams in type (e.g., sudden, gradual, incremental). 
\end{challenge}

\section{Methodology}
\label{ssec:methodology}
We present the \mymethod~framework to address the three fundamental challenges in HML. The core innovation lies in a drift-aware autonomous architecture that combines stream-specific specialization with cross-stream collaboration.

\subsection{Overview of~\mymethod}
As shown in Figure~\ref{fig:architecture}, we introduce \mymethod, a framework designed to learn a generalized model $\mathcal{F}$ by processing $n$ data streams $\mathcal{S}$ in a window-based prequential manner. It features a modular architecture where each stream $\mathcal{S}_i$ is assigned a dedicated learning system, including:
\begin{itemize}
    \item A stream-specific Feature Extractor ($\texttt{FE}_i$) for dimensionality and feature space alignment.
    \item A Mixture of Experts (MoE) core, which includes a dynamic pool of Private Experts ($\texttt{PE}_i$), a dedicated Assistance Expert ($\texttt{AE}_i$), and a Routing Network ($\texttt{RN}_i$).
    \item A task-specific Classification Head ($\texttt{CH}_i$) for handling heterogeneous label spaces.
    \item A control loop comprising a Drift Detector ($\texttt{DD}_i$) and an Autonomous Expert Tuner ($\texttt{AET}_i$) for online adaptation.
\end{itemize}

The online learning process begins with an initial model trained on the first window $W_0$. Subsequently, for each incoming data window $W_t$ ($t \geq 1$), the system executes a \textbf{Test-Diagnose-Adapt} cycle, and the whole process is summarized in Algorithm~\ref{alg:fullProcess}.

\paragraph{Phase \filledcircled[\small]{1} -- Test and Record.}
The cycle begins by evaluating the current model state $\mathcal{F}_{t-1}$ (trained on $W_{t-1}$) on the new data of window $W_t$. For each instance $(\boldsymbol{x}_{i,t}, y_{i,t})$, its feature vector $\boldsymbol{x}_{i,t}$ is first projected by $\texttt{FE}_{i, t-1}$ to an aligned representation $\boldsymbol{h}_{i,t}$. It is then processed by the dynamic MoE. Concretely, the $\texttt{RN}_{i, t-1}$ computes routing weights to combine outputs from the $\texttt{PE}_i(t-1)$ pool, which captures idiosyncratic patterns, and the $\texttt{AE}_{i, t-1}$, which performs collaborative fusion by attending to features $\{\boldsymbol{h}_{j,t}\}_{j \neq i}$ from all other streams. The resulting integrated feature vector is finally passed to the task-specific $\texttt{CH}_{i, t-1}$ to produce a prediction $\hat{y}_{i,t}$. The performance (e.g., accuracy) against the true label $y_{i,t}$ is then recorded for the subsequent phase.

\paragraph{Phase \filledcircled[\small]{2} -- Diagnose and Decide.}
Following the test, the system diagnoses the state of each stream by the drift detector $\texttt{DD}_i$. It analyzes the distribution of features $\{\boldsymbol{h}_{i,t}\}$ from $W_t$ to detect drift in $P_t^{(i)}$. Concurrently, an autonomous expert tuner $\texttt{AET}_{i,t}$ evaluates the performance metrics from the test phase and the long-term utilization statistics of its experts. Based on these evidences, i.e., the drift signal and performance analysis, the $\texttt{AET}_{i,t}$ makes an adaptation decision: it may expand the private expert pool $\texttt{PE}_i(t)$ by adding a new expert to learn an emerging concept, or prune an underutilized expert to maintain model parsimony.

\paragraph{Phase \filledcircled[\small]{3} -- Adapt and Train.}
Finally, the model architecture is updated based on the decisions from the diagnosis phase. The potentially modified model is then trained on the data from window $W_t$ via an end-to-end process. The total loss aggregated from all stream-specific classification heads is back-propagated through the entire network. This step refines the parameters of all active components, preparing the system for the next window $W_{t+1}$. This cyclical process allows \mymethod~to continuously learn, adapt, and specialize in a non-stationary multistream environment.

\begin{algorithm}[th]
\caption{\mymethod: Online Learning Process}
\label{alg:fullProcess}
\begin{algorithmic}[1]
\REQUIRE Data streams $\{\mathcal{S}_i\}_{i=1}^n$, Window size $|W_{i}|$, Total windows $T_{max}$.
\ENSURE Predicted labels.
\STATE \textcolor{gray}{\% Initial training on the first window $W_0$}
\STATE $W_{0} \in \{\mathbf{X}_{i, 0}, Y_{i, 0}\}_{i=1}^{n} $  $ \leftarrow $ $\text{GetData}(\{S_i\}_{i=1}^{n}, |W_{i}|, 0)$.
\STATE $\mathcal{F}_0 \leftarrow \text{Train}(W_{0})$. 
\item[] \textcolor{gray}{\% Test-then-Adapt loop for subsequent windows.}
\FOR{$t = 1:T_{max}$}
    \STATE $W_{t} \in \{\mathbf{X}_{i, t}, Y_{i, t}\}_{i=1}^{n} \leftarrow \text{GetWindow}(\{S_i\}_{i=1}^{n}, l, t)$
\item[] \textcolor{gray}{\% Phase \filledcircled[\small]{1} -- Test \& performance record.}
    \STATE $\{\texttt{Perf}_{i,t}\}_{i=1}^n \leftarrow \text{Test}(\mathcal{F}_{t-1}, W_{t})$
    \STATE Update $\{\texttt{AET}_{i,t}\}$ with stream-specific performances $\{\texttt{Perf}_{i,t}\}$.
\item[] \textcolor{gray}{\% Phase \filledcircled[\small]{2} -- Diagnose \& Decide.}
    \FOR{$i = 1: n$}
        \STATE $\texttt{drift\_signal}_{i,t} \leftarrow \texttt{DD}_{i,t}.\text{update}(H_{i,t})$
        \STATE $(\texttt{action}_{i,t}$, $\texttt{pe\_id}_{i,t})$  $\leftarrow$  $\texttt{AET}_{i,t}(\texttt{drift\_signal}_{i,t}, \texttt{Perf}_{i,t} )$
        \IF{$ \texttt{action}_{i,t}$ is "\texttt{ADD\_PRIVATE}"}
            \STATE Add a new private expert to $\texttt{PE}_i(t)$;
            \STATE Freeze old private experts; 
        \ELSIF{$\texttt{action}_{i,t}$ is "\texttt{PRUNE\_PRIVATE}"}
            \STATE Prune Private Experts $\texttt{PE}_i(t)[\texttt{pe\_id}]$ 
        \ENDIF
    \ENDFOR  
\item[] \textcolor{gray}{\% Phase \filledcircled[\small]{3} -- Adapt \& Train.}
    \STATE $\mathcal{F}_t$ $\leftarrow $$ \text{Train}(W_t)$
    
\ENDFOR
\end{algorithmic}
\end{algorithm}

\subsection{Heterogeneity-Aware Representation}
\label{ssec:tackling_heterogeneity}
To handle the \textit{intrinsic heterogeneity, i.e., Challenge~\ref{cha:heterogeneity}}, our framework employs a hierarchical approach involving feature-level alignment and task-level specialization.

\noindent \textit{1. Feature Alignment.} 
First, to address the feature space heterogeneity ($\mathcal{X}_i \neq \mathcal{X}_j$), each stream $\mathcal{S}_i$ is assigned a dedicated feature extractor $\texttt{FE}_i$. This is a neural network, parameterized by $\theta_{\texttt{FE}}^{(i,t)}$, whose architecture is tailored to the input dimensionality $D_i$. Its primary function is to project the raw feature vector $\boldsymbol{x}_{i,t}$ into a common latent space $\mathcal{H} \subset \mathbb{R}^{D_h}$:
\begin{equation}
    \boldsymbol{h}_{i,t} = \texttt{FE}_i(\boldsymbol{x}_{i,t},  \theta_{\texttt{FE}}^{(i,t)}), i \in [0, n];
\end{equation}
This explicit dimensionality alignment creates a standardized input format for all subsequent expert networks, forming the foundation for inter-stream knowledge fusion.

\noindent
\textit{2. Task Specialization.}
Second, to address the label space heterogeneity ($\mathcal{Y}_i \neq \mathcal{Y}_j$), our framework adopts a multi-task learning paradigm. Each stream $\mathcal{S}_i$ is equipped with an independent task-specific classification head $\texttt{CH}_i$ parameterized by $\theta_{\texttt{CH}}^{(i,t)}$. It is responsible for mapping the final refined feature representation $\boldsymbol{f}_{i,t}$ (derived by Eq.~\eqref{eq:routingFusion}) to the stream's unique label space $\mathcal{Y}_i$:
\begin{equation}
    Logits_{i,t} = \texttt{CH}_{i}(\boldsymbol{f}_{i,t}; \theta_{\texttt{CH}}^{(i,t)}) \in \mathbb{R}^{C_i}, i \in [0, n];
\end{equation}
This architecture ensures that the final decision-making process is tailored to each stream's specific predictive task, whether it is binary classification or multi-class classification with a different number of classes.

\subsection{Adaptive Knowledge Fusion}
\label{ssec:knowledge_fusion_mechanism}
To address the \textit{Knowledge Fusion, i.e., Challenge~\ref{cha:fusion}}, \mymethod~introduces a novel dynamic MoE architecture with collaborative assistance designed to exploit inter-stream correlations while mitigating negative transfer.

\noindent
\textit{1. Private Experts: capturing stream-specific knowledge.}
For each stream $\mathcal{S}_i$, we maintain a dynamic pool of private experts $\texttt{PE}_i(t) = \{pe_{i,j} | j=1, \dots, K_i(t)\}$. Each expert $\texttt{pe}_{i,j}$ is an MLP parameterized by $\theta_{\texttt{pe}}^{(i,j,t)}$, that learns patterns idiosyncratic to stream $\mathcal{S}_i$. It processes the aligned feature $\boldsymbol{h}_{i,t}$ to produce representations in a common expert output space $\mathcal{E} \subset \mathbb{R}^{D_f}$:
\begin{equation}
    \boldsymbol{f}_{i,j,t}^{\texttt{pe}} = \texttt{pe}_{i,j}(\boldsymbol{h}_{i,t}; \theta_{\texttt{pe}}^{(i,j,t)})
\end{equation}

\noindent
\textit{2. Assistance Experts: collaborative knowledge fusion.}
Each stream $\mathcal{S}_i$ is paired with a dedicated assistance expert $\texttt{AE}_i$ parameterized by $\theta_{\texttt{AE}}^{(i,t)}$. This expert's unique role is to perform collaborative knowledge fusion. It takes the target stream's feature $\boldsymbol{h}_{i,t}$ as a \textit{query} and leverages features from all other concurrent streams $\{\boldsymbol{h}_{j,t}\}_{j \neq i}$ as \textit{context} (keys and values)~\cite{vaswani2017attention,zhang2025multimodal}. We employ a multi-head attention mechanism:
\begin{equation}
    \boldsymbol{c}_{i,t} = \text{Attention}(\boldsymbol{h}_{i,t}, \{\boldsymbol{h}_{j,t}\}_{j \neq i})
\end{equation}
The resulting context vector $\boldsymbol{c}_{i,t} \in \mathbb{R}^{D_h}$ is a weighted summary of information from other streams, where the weights are learned based on relevance to $\boldsymbol{h}_{i,t}$. This contextual information is then fused with the input features to produce the assistance expert's output representations:
\begin{equation}
    \boldsymbol{f}_{i,t}^{\texttt{AE}} = \text{MLP}_{\texttt{AE}}^{(i,t)}(\text{Concat}(\boldsymbol{h}_{i,t}, \boldsymbol{c}_{i,t}); \theta_{\texttt{AE}}^{(i,t)}) \in \mathbb{R}^{D_f}
\end{equation}
This end-to-end mechanism allows $\texttt{AE}_{i}$ to learn \textit{what} information to transfer from other streams and \textit{how} to use it to best serve stream $\mathcal{S}_i$.

\noindent
\textit{3. Routing and feature integration}.
A stream-specific routing network $\texttt{RN}_{i}$ parameterized by $\theta_{\texttt{RN}}^{(i,t)}$ determines the credibility of each expert for a given input $\boldsymbol{h}_{i,t}$. It outputs a probability distribution $\boldsymbol{p}_{i,t}$ over the $K_i(t)$ private experts and the assistance expert. The final refined representations $\boldsymbol{f}_{i,t}$ for stream $\mathcal{S}_i$ are a weighted combination of all expert outputs:
\begin{equation}
    \boldsymbol{f}_{i,t} = \boldsymbol{p}_{i,t}[\texttt{AE}_{i}] \cdot \boldsymbol{f}_{i,t}^{\texttt{AE}} + \sum_{j=1}^{K_i(t)} \boldsymbol{p}_{i,t}[\texttt{pe}_{i,j}] \cdot \boldsymbol{f}_{i,j,t}^{\texttt{pe}}
    \label{eq:routingFusion}
\end{equation}
The routing mechanism provides a natural defense against negative transfer as it can learn to assign a near-zero weight to the assistance expert if the external context is irrelevant or even harmful.

\subsection{Drift Detection \& Adaptation}
\label{ssec:online_adaptation_drift}
Our framework's autonomy and ability to handle \textit{asynchronous drifts (Challenge~\ref{cha:asyn})} stem from a per-stream control loop involving a drift detector and an expert tuner.

\noindent
\textit{1. Drift Detection.}
Each stream $\mathcal{S}_i$ is independently monitored by a Maximum Mean Discrepancy (MMD) based drift detector $\texttt{DD}_i$~\cite{wan2024online}. $\texttt{DD}_i$ maintains a reference window $W^{ref}_{i,t}$ of past features $\boldsymbol{h}_s$ and compares it with the features from the current window $W_{i,t}$.
\begin{equation}
\small
\begin{aligned}
    & MMD_{i,t}^2(W_{i,t}, W^{ref}_{i,t}) = \\
   & \left\| \frac{1}{|W_{i,t}|} \sum_{\boldsymbol{h}_{i,t} \in W_{i,t}} \phi(\boldsymbol{h}_{i,t}) - \frac{1}{|W^{ref}_{i,t}|} \sum_{\boldsymbol{h}_{j,t} \in W^{ref}_{i,t}} \phi(\boldsymbol{h}_{j,t}) \right\|_{\mathcal{H}_k}^2
\end{aligned}
\end{equation}
where $\phi$ is a mapping to a Reproducing Kernel Hilbert Space $\mathcal{H}_k$ induced by a kernel~\cite{smola2007hilbert}. If $MMD_{i,t}^2 > \tau_{MMD_i}$, $\texttt{DD}_i$ signals a drift for stream $\mathcal{S}_i$. The reference window $W^{ref}_{i,t}$ is then updated with $W_{i,t}$.

\noindent
\textit{2. Autonomous Expert Tuner.}
To achieve robust and efficient adaptation, our framework employs an Autonomous Expert Tuner ($\texttt{AET}_i$) that governs the lifecycle of private experts for each stream $\mathcal{S}_i$. Relying solely on distribution-based drift detection ($\texttt{DD}_i$) can be suboptimal, as not all statistical shifts necessarily degrade predictive performance~\cite{lu2018learning}, which could lead to unnecessary and costly model adaptations. Conversely, some performance degradation might occur without a detectable distribution shift in the feature space. Therefore, the $\texttt{AET}_i$ integrates two complementary signals, i.e., the drift signal from $\texttt{DD}_i$ and the stream's recent test performance. This expert-level plasticity is the core mechanism for adapting model capacity online:
\begin{itemize}
    \item \textbf{Expert Adding:} A new private expert is added to the pool $\texttt{PE}_i$ only when a drift is detected by $\texttt{DD}_i$ \textbf{and} the stream's test performance $\texttt{Perf}_{i,t}$ exhibits a significant degradation. This conjunctive condition ensures that the model only expands its capacity when there is clear evidence of a detrimental concept change. The new expert is initialized as trainable to learn the emerging concept, while all existing private experts in $\texttt{PE}_i$ are frozen to prevent catastrophic forgetting, thereby preserving knowledge of past concepts.

    \item \textbf{Expert Pruning:} A private expert $\texttt{pe}_{i,j}$ (whether frozen or active) is pruned from $\texttt{PE}_i$ if its long-term average utilization, determined by the routing weights from $\texttt{RN}_i$, falls below a threshold $\tau_{util}$. This proactive mechanism removes irrelevant experts that no longer contribute to the stream's predictions, maintaining model parsimony and preventing the accumulation of obsolete components.
\end{itemize}
Since each $\texttt{AET}_i$ operates independently based on its stream's specific signals, the framework naturally handles asynchronous drifts.

\subsection{Learning Objective}
\label{subsec:optimization}
This method is trained end-to-end. For a given data window $W_t$, the total loss is the sum of the individual cross-entropy losses from each stream-specific classification head:
\begin{equation}
\small
    L_{total}(W_t) = \sum_{i=1}^{n} \mathbb{E}_{(\boldsymbol{x}_{i,j}, y_{i,j}) \in W_{t,i}} \left[ \mathcal{L}_{\texttt{CE}}(\texttt{CH}_i( \boldsymbol{f}_{i,t}), y_{i,t}) \right],
\end{equation}

\begin{table*}[thbp]
  \centering
  \small
  \setlength{\tabcolsep}{3pt}
    \begin{tabular}{cccc>{\bfseries}cccc>{\bfseries}cccc>{\bfseries}cccc>{\bfseries}c}
    \toprule
    \multirow{2}[4]{*}{\textbf{Synthetic}} & \multicolumn{4}{c}{Set 1: Tree (Homo.)} & \multicolumn{4}{c}{Set 2 Hyperplane (Homo.)} & \multicolumn{4}{c}{Set 3 (Hete.)} & \multicolumn{4}{c}{Set 4 (Hete.)} \\
    \cmidrule(lr){2-5}  \cmidrule(lr){6-9}  \cmidrule(lr){10-13}  \cmidrule(lr){14-17}       
     &$\mathcal{S}_1$  &$\mathcal{S}_2$  & $\mathcal{S}_3$ & avg & $\mathcal{S}_1$  & $\mathcal{S}_2$  & $\mathcal{S}_3$ & avg  & SEAa  & RTG   & RBF & avg & LED & LEDDri & Wave &avg\\
    \midrule
    SRP  &  58.47     &  65.14    &  64.63 & 62.74    &  86.37      & 87.59       & 88.21     & 87.39     & 83.35       & \textcolor{red}{70.05}      & 81.18   & 78.19     & 35.18       &\textcolor{red}{36.65}       &\textcolor{blue}{83.80}  &51.88\\
    AMF & 56.18     &  63.76     & 59.59  & 59.84   &  \textcolor{blue}{91.32}     & 90.70      & 90.70     & \textcolor{blue}{90.91} & 83.65       & 66.19       & 90.29  &  80.04   & 37.85      & 25.31       &79.39 & 47.52 \\
    IWE   & 63.49      &  \textcolor{red}{72.35}    & \textcolor{red}{68.39}   & \textcolor{red}{68.07}     &  89.82      &  \textcolor{blue}{91.39}    &  \textcolor{blue}{90.90}    &90.70  & 84.27      & 64.38      & 70.12  & 72.92    & 36.05      & 34.15      & 80.41  & 50.20\\
    MCMO &  64.77     & 67.29  & 66.32   &  66.13     &   82.21    & 85.37 &   85.12   & 84.23    &   -    &   -    &  -  &  - &   -  &   -   &- & -\\
    OBAL  & 65.72  &  67.97  & 65.60  &   66.43    & 84.14      &  86.73     &  88.66     &  86.51 &   -    &   -    &  -  &  - &   -  &   -   & - & - \\
    BFSRL  & 63.37   & 67.42  & 64.39    & 65.06 &84.67    & 87.20  &  88.47     & 86.78         &   -    &  -     &  -  &  - &   -  &   -   &  -  & - \\
   \mymethod  &  \textcolor{red}{65.78}     &  \textcolor{blue}{ 68.27}    & \textcolor{blue}{66.48} & \textcolor{blue}{66.84} &   \textcolor{red}{91.85}  &   \textcolor{red}{92.12}   &   \textcolor{red}{91.84}   & \textcolor{red}{91.94}   & \textcolor{red}{85.14}      &\textcolor{blue}{67.73}     &\textcolor{red}{92.75}   &\textcolor{red}{81.87}    &\textcolor{red}{38.19}     &\textcolor{blue}{35.36}       &\textcolor{red}{85.43}  &\textcolor{red}{53.00} \\
    \midrule
    \midrule
    \multirow{2}[4]{*}{\textbf{Real-World}} & \multicolumn{4}{c}{Set 5: TV News (Homo.)} & \multicolumn{4}{c}{Set 6: Weather (Homo.)} & \multicolumn{4}{c}{Set 7:  Credit card (Hete.)} & \multicolumn{4}{c}{Set 8: CoverT. (Hete.)} \\
    \cmidrule(lr){2-5}  \cmidrule(lr){6-9}  \cmidrule(lr){10-13}  \cmidrule(lr){14-17}      
    & CNN  & BBC & TIMES &avg  & $\mathcal{S}_1$ &  $\mathcal{S}_2$ & $\mathcal{S}_3$ & avg & $\mathcal{S}_1$ & $\mathcal{S}_2$ & $\mathcal{S}_3$  &avg & $\mathcal{S}_1$ & $\mathcal{S}_2$ & $\mathcal{S}_3$ &avg \\
    \midrule
    SRP   &  78.46     &   75.55    & \textcolor{blue}{80.84}  & 78.28 &  \textcolor{blue}{81.46}    &   \textcolor{blue}{77.45}    &   \textcolor{blue}{78.15}  & \textcolor{blue}{79.02} & 77.81      &  \textcolor{red}{82.01}     &  78.04  & \textcolor{blue}{79.29 }    &  \textcolor{red}{87.21}    & 52.99       & 56.36   & 65.52  \\
    AMF  &  \textcolor{blue}{79.25} & \textcolor{blue}{79.49}       &  78.70     &  \textcolor{blue}{79.15}    &  81.37    &  75.70      & 77.91  & 78.33 & \textcolor{blue}{77.86}       &81.40       & 77.88  & 78.39    &  86.15      & 53.62       & 61.75   & \textcolor{blue}{67.17}   \\
    IWE  &  78.66     & 74.42      &77.54 &   76.87  &    80.40   &  76.24     & 74.91   & 77.18 & 75.89      & 80.15      & 75.92   &  77.32    & 72.58      & 51.52      &51.75    & 58.62   \\
    MCMO  &   68.83    & 60.12      & 59.74   &  62.90  &  75.11     & 75.02    & 73.37      & 74.50  &   -    &   -    &  -     &   -  &- & -& -&  -        \\
    OBAL  &   67.72    & 59.39      & 64.42  &  63.84 &  77.46    & 74.35    & 76.21      & 75.97  &  -     &   -    &    -   &   - &-  &-   &-   &         - \\
    BFSRL  &  60.18     &   55.09    & 61.29     & 59.12   & 74.77   & 74.09  & 75.42   & 74.76 &   -    &   -    &    -   &    -   &    -   & -      &     -  &  -       \\
    \mymethod  & \textcolor{red}{80.06}  & \textcolor{red}{79.66}  &\textcolor{red}{80.90}   & \textcolor{red}{80.21}  & \textcolor{red}{82.04} & \textcolor{red}{78.33}       &  \textcolor{red}{79.39}        & \textcolor{red}{79.92}&  \textcolor{red}{80.42}  & \textcolor{blue}{81.93}   & \textcolor{red}{80.37}   &  \textcolor{red}{80.91}     &  \textcolor{blue}{86.97}     &  \textcolor{red}{62.91}     &\textcolor{red}{82.22}  &  \textcolor{red}{77.37}   \\
    \bottomrule
    \end{tabular}
    \caption{Classification accuracy (\%) of various methods on all benchmarks. The best and second-best results are highlighted in red and blue respectively. "-" means it is not applicable to the task.}
  \label{tab:overall}%
\end{table*}%

\subsection{Theoretical Analysis}
\label{ssec:theory}
The design of our method is theoretically grounded in multi-task learning principles~\cite{maurer2016benefit}, which demonstrates that jointly learning related tasks can yield superior generalization over isolated learning. Our collaborative assistance mechanism enables intelligent knowledge fusion while mitigating negative transfer, and can be formally justified by:
\begin{theorem}[Generalization Bound]
\label{thm:gene_bound}
Let $\mathcal{F}$ be the hypothesis space defined by the CAMEL architecture, for any hypothesis $h \in \mathcal{F}$ trained on streams $\mathcal{S}=\{S_i\}_{i=1}^n$, the expected risk $\mathcal{R}_i(h)$ on any stream $S_i$ is bounded as:
\begin{equation} 
\small
\label{eq:main_bound_simple}
    \mathcal{R}_i(h) \leq \hat{\mathcal{R}}_{\text{avg}}(h) + C(\{\mathcal{S}_j\}_{j=1}^n) + \mathcal{O}\left(\sqrt{\frac{ \log(n|W_{t}|)}{n|W_{t}|}}\right)
\end{equation}
where $\hat{\mathcal{R}}_{\text{avg}}(h) = \frac{1}{n}\sum_{j=1}^n \hat{\mathcal{R}}_j(h)$ is the average empirical risk across all streams and $C(\{\mathcal{S}_j\})$ quantifies the inter-stream dissimilarity. A proof sketch is in Appendix A.
\end{theorem}

\begin{implications}
Theorem~\ref{thm:gene_bound} formally justifies \mymethod's architecture: The assistance expert ($\texttt{AE}_i$) minimizes $C(\{\mathcal{S}_i\}_{i=1}^n)$ through attention-based knowledge transfer, while the routing network ($\texttt{RN}_i$) dynamically balances this against stream-specific private experts ($\texttt{PE}_i$) to prevent negative transfer when dissimilarity is high. This intrinsic collaboration-specialization tradeoff combined with joint training's sample efficiency ($\mathcal{O}(1/\sqrt{n|W_{t}|})$) explains the empirical robustness. The autonomous expert tuner (AET) maintains adaptability to concept drift across windows through expert-level plasticity.
\end{implications}

\section{Experiments}
\label{sec:experiments}
In experiments, we first assess the framework's \textit{generality and robustness} across both homogeneous and heterogeneous settings. Second, we provide a qualitative analysis of the \textit{online adaptation process} visualizing how the $\texttt{AET}$ dynamically manages the private expert pool to concept drifts. Finally, we perform an ablation study to dissect the contribution of each core component, thereby validating our fundamental design principles. More detailed analysis and supplementary experiments can be seen in Appendix C.

\subsection{Experiment Settings}
\subsubsection{Benchmarks.}
We establish eight diverse multistream scenarios. The first four scenarios are constructed from twelve synthetic data streams, meticulously designed to isolate specific challenges: homogeneous (Set 1 \& 2) and heterogeneous (Set 3) feature spaces, and heterogeneous label spaces (Set 4). In addition, we employ four real-world multistream datasets, which inherently exhibit a mix of homogeneous and heterogeneous characteristics (Set 5-8). More detailed descriptions can be found in Appendix B.1.

\subsubsection{Baselines.}
We conduct a comparison against six SOTA methods, including \emph{1) Single-stream learning}: SRP~\cite{gomes2019streaming}, AMF~\cite{mourtada2021amf} and IWE~\cite{jiao2022incremental}; \emph{2) Multistream classification}: MCMO~\cite{jiao2023reduced}, OBAL~\cite{yu2024online} and BFSRL~\cite{yu2024fuzzy}. The detailed description and implementation are provided in Appendix B.2 \& B.3.
\subsection{Results Analysis}

\begin{table*}[thbp]
  \centering
  \small
  \setlength{\tabcolsep}{3pt}
    \begin{tabular}{cccc>{\bfseries}cccc>{\bfseries}cccc>{\bfseries}cccc>{\bfseries}c}
    \toprule
    \multirow{2}[4]{*}{\textbf{Variants}} & \multicolumn{4}{c}{Set 3} & \multicolumn{4}{c}{Set 4 } & \multicolumn{4}{c}{Set 6: Weather} & \multicolumn{4}{c}{Set 7: Credit Card} \\
    \cmidrule(lr){2-5}  \cmidrule(lr){6-9}  \cmidrule(lr){10-13}  \cmidrule(lr){14-17}       
     & SEAa  & RTG   & RBF & avg  & LED & LEDDri & Wave &avg &$\mathcal{S}_1$  &$\mathcal{S}_2$  & $\mathcal{S}_3$ & avg & $\mathcal{S}_1$  & $\mathcal{S}_2$  & $\mathcal{S}_3$ & avg   \\
    \midrule
    Base  &  80.28      & 64.37     &  81.42  & 75.36    & 29.21       &  21.96     & 76.94    & 42.70 & 74.32     & 73.17     &  73.21 &  73.57   & 74.29       &  77.24    & 74.02 & 75.18\\
    Base+I &  83.32     &  66.10     & 87.27   &  78.90  & \textcolor{blue}{37.31}      &  34.13     &  \textcolor{blue}{83.22}    & \textcolor{blue}{51.55} &  76.22     &  77.04      &   76.78   & 76.68 & 77.92    & 76.58      & 76.74 & 77.08 \\
    Base+I+DP &  \textcolor{blue}{84.84}     & \textcolor{blue}{66.17}     & \textcolor{blue}{89.16}   & \textcolor{blue}{80.06}  & 37.23       & \textcolor{blue}{34.30}     &  82.97    & 51.50  & \textcolor{blue}{79.74}      &  \textcolor{blue}{77.67}    & \textcolor{blue}{78.01}  &  \textcolor{blue}{78.47}   & \textcolor{blue}{78.63}      &  \textcolor{red}{82.07}      & \textcolor{blue}{79.09}  & 79.93\\
   \mymethod & \textcolor{red}{85.14}      &\textcolor{red}{67.73}     &\textcolor{red}{92.75}   &\textcolor{red}{81.87}  &\textcolor{red}{38.19}     &\textcolor{red}{35.36}       &\textcolor{red}{85.43}  &\textcolor{red}{53.00}  & 82.04 & 78.33       &  \textcolor{blue}{79.39}        & \textcolor{red}{79.92}   & \textcolor{red}{80.42}  & \textcolor{blue}{81.93}   & \textcolor{red}{80.37}   &  \textcolor{red}{80.91}        \\
    \bottomrule
    \end{tabular}
    \caption{Ablation study. Classification accuracy (\%) of \mymethod's variants. The best and second-best results are highlighted in red and blue, respectively.}
  \label{tab:ablation}%
\end{table*}%

\subsubsection{Overall Performance.}
Table~\ref{tab:overall} demonstrates that \mymethod~consistently achieves SOTA average accuracy across almost all scenarios except for Set 1, validating its strong generality and robustness. The framework's primary strength lies in its effective handling of the \textit{Intrinsic Heterogeneity (Challenge 1)}. Unlike contemporary multistream methods (MCMO, OBAL, BFSRL) which are confined to homogeneous settings and thus not applicable to our more realistic heterogeneous scenarios, our method thrives in these complex environments. This is enabled by its stream-specific modules ($\texttt{FE}_i$, $\texttt{CH}_i$), which provide the necessary specialization for each stream. Furthermore, compared against single-stream methods (SRP, AMF, IWE), \mymethod's consistent top-tier performance validates its novel approach to the \textit{Knowledge Fusion (Challenge 2)}. While single-stream methods operate in isolation, our collaborative assistance mechanism successfully leverages latent inter-stream correlations. The attention-based experts perform targeted knowledge transfer, boosting the overall system performance. This dynamic interplay between specialized private experts managed by the $\texttt{AET}_i$ to address \textit{Asynchronous Drifts (Challenge 3)}, and the collaborative assistance experts allows it to strike a robust balance between focused learning and knowledge fusion. Consequently, our method excels across the full spectrum of HML challenges, proving its capability as a general and powerful solution for diverse and evolving multistream environments.

\subsubsection{Online Performance.}
Figure~\ref{fig:online} qualitatively analyzes \mymethod's online adaptation, plotting per-stream accuracy against the number of private experts. The results illustrate the 'drift-diagnose-adapt' narrative and validate the Autonomous Expert Tuner ($\texttt{AET}$). For example, in Figure~\ref{fig:online_sea}, Stream 1 exhibits an accuracy dip at window 15, indicating concept drift. The $\texttt{AET}$ correctly diagnoses this and responds by instantiating a new private expert, increasing model capacity and enabling swift performance recovery. Once the new concept is learned, the redundant expert is pruned (around window 20) to maintain model parsimony. Conversely, Stream 2 (without significant drift) demonstrates $\texttt{AET}$'s robustness: despite accuracy fluctuations, the private expert count remains constant, showing it avoids overreacting to inherent data noise (similar to Figure~\ref{fig:online_led}). These behaviors highlight that CAMEL's adaptation is highly selective, providing architectural plasticity precisely when and where needed to autonomously maintain high performance amidst asynchronous concept drifts. Additional visualizations are in Appendix C.1.

\begin{figure}[ht]
    \centering
    \begin{subfigure}[t]{\columnwidth}
        \centering
      \includegraphics[width=\linewidth]{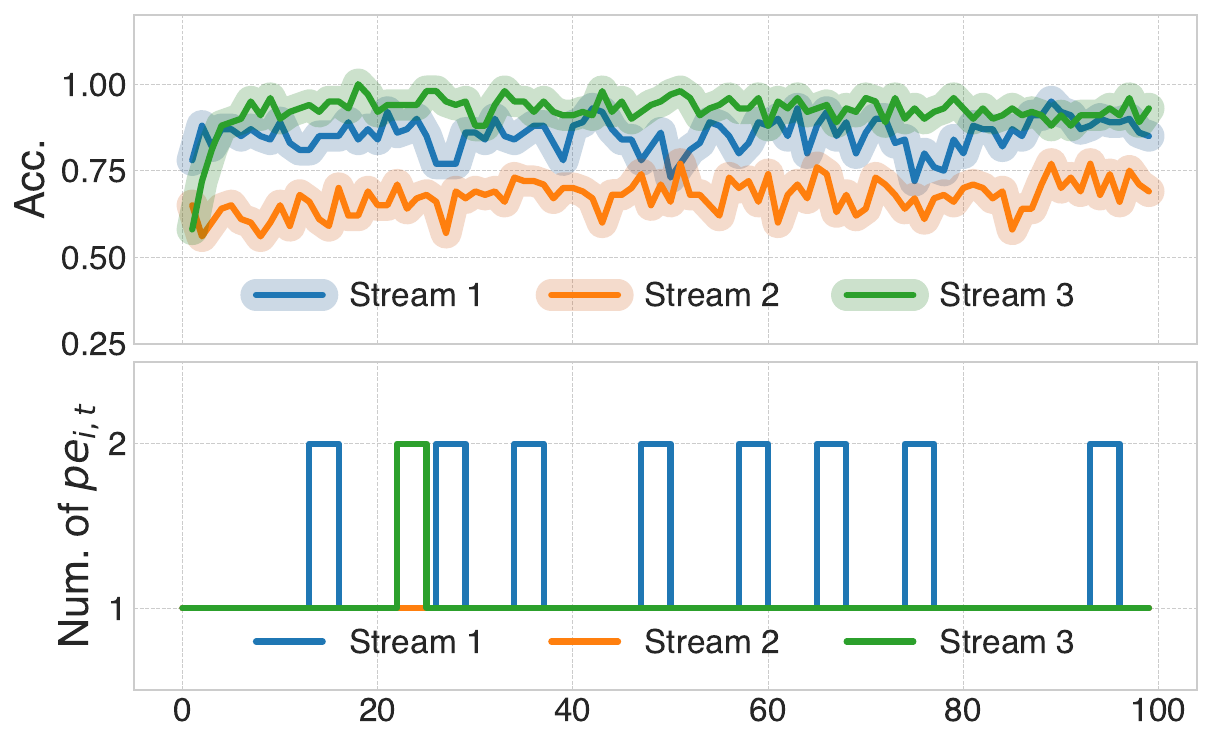}
        \caption{Set 3: SEAa, RTG, RBF.}
        \label{fig:online_sea}
    \end{subfigure}
    \begin{subfigure}[t]{\columnwidth}
        \centering
        \includegraphics[width=\linewidth]{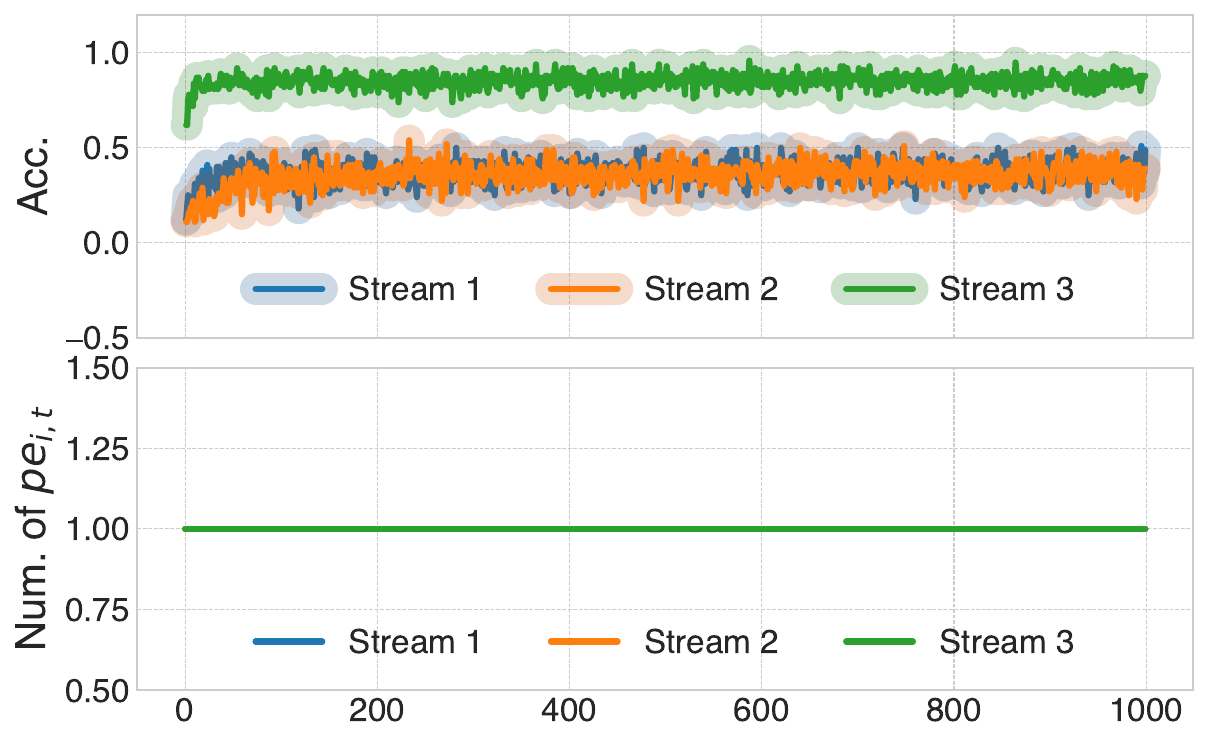}
        \caption{Set 4: LED, LEDDrift, Waveform}
        \label{fig:online_led}
    \end{subfigure}
    \caption{Online accuracy and the corresponding number of private experts over time.}
    \label{fig:online}
\end{figure}

\subsubsection{Ablation Study.}
To dissect component contributions, our ablation study progressively constructs the full \mymethod~framework (Table~\ref{tab:ablation}) validating core design principles. Transitioning from the naive \textbf{Base} (full retraining) to \textbf{Base+I} (incremental learning) yields significant gains, confirming continuous fine-tuning mitigates catastrophic forgetting. Integrating the Autonomous Expert Tuner (\textbf{Base+I+DP}) further improves performance on drifting streams (e.g., RBF: +1.89\%), demonstrating expert-level plasticity effectively addresses \textit{Asynchronous Drifts (Challenge 3)}. The full \mymethod~framework with collaborative assistance delivers the most substantial improvement, which empirically shows the attention-based mechanism masters \textit{Knowledge Fusion (Challenge 2)} by distilling cross-stream knowledge for superior HML generalization.

\section{Conclusion \& Limitation}
In this paper, we introduced \mymethod, a novel autonomous Mixture of Experts framework designed to robustly handle the complexities of multistream learning. By assigning each stream a dynamic ensemble of specialized private experts alongside a dedicated collaborative assistance expert, our method effectively addresses intrinsic heterogeneity and facilitates adaptive knowledge fusion. In addition, an autonomous tuner manages the expert lifecycle at a modular level and allows our method to dynamically adapt to concept drifts. A generalization bound based on multi-task learning theory formally connects inter-stream relatedness and routing decisions with task-level risk. Empirical results on diverse synthetic and real-world multistream settings demonstrate the superiority under HML challenges.

Limitations include suboptimal handling of recurring concepts through expert freezing and computational overhead from dynamic architecture adaptation. Future work will explore expert reactivation strategies and efficiency optimizations for resource-constrained environments.

\section{Acknowledgments}
The work was supported by the Australian Research Council (ARC) under Laureate project FL190100149 and discovery project DP220102635.
\bibliography{aaai2026}

\begin{thebibliography}{46}
\providecommand{\natexlab}[1]{#1}

\bibitem[{Agrahari and Singh(2022)}]{agrahari2022concept}
Agrahari, S.; and Singh, A.~K. 2022.
\newblock Concept drift detection in data stream mining: A literature review.
\newblock \emph{Journal of King Saud University-Computer and Information
  Sciences}, 34(10): 9523--9540.

\bibitem[{Bifet and Gavalda(2007)}]{bifet2007learning}
Bifet, A.; and Gavalda, R. 2007.
\newblock Learning from time-changing data with adaptive windowing.
\newblock In \emph{Proceedings of the 2007 SIAM international conference on
  data mining}, 443--448. SIAM.

\bibitem[{Brzezinski and Stefanowski(2013)}]{brzezinski2013reacting}
Brzezinski, D.; and Stefanowski, J. 2013.
\newblock Reacting to different types of concept drift: The accuracy updated
  ensemble algorithm.
\newblock \emph{IEEE transactions on neural networks and learning systems},
  25(1): 81--94.

\bibitem[{Cacciarelli and Kulahci(2024)}]{cacciarelli2024active}
Cacciarelli, D.; and Kulahci, M. 2024.
\newblock Active learning for data streams: a survey.
\newblock \emph{Machine Learning}, 113(1): 185--239.

\bibitem[{Ditzler and Polikar(2012)}]{ditzler2012incremental}
Ditzler, G.; and Polikar, R. 2012.
\newblock Incremental learning of concept drift from streaming imbalanced data.
\newblock \emph{IEEE transactions on knowledge and data engineering}, 25(10):
  2283--2301.

\bibitem[{Domingos and Hulten(2000)}]{Domingos2000}
Domingos, P.; and Hulten, G. 2000.
\newblock Mining high-speed data streams.
\newblock In \emph{Proceedings of the 6th ACM SIGKDD International Conference
  on Knowledge Discovery and Data Mining}, 71--80. New York, USA: ACM.

\bibitem[{Gomes, Read, and Bifet(2019)}]{gomes2019streaming}
Gomes, H.~M.; Read, J.; and Bifet, A. 2019.
\newblock Streaming random patches for evolving data stream classification.
\newblock In \emph{2019 IEEE International Conference on Data Mining},
  240--249. IEEE.

\bibitem[{Guo, Zhang, and Wang(2021)}]{guo2021selective}
Guo, H.; Zhang, S.; and Wang, W. 2021.
\newblock Selective ensemble-based online adaptive deep neural networks for
  streaming data with concept drift.
\newblock \emph{Neural Networks}, 142: 437--456.

\bibitem[{Jiao et~al.(2024)Jiao, Guo, Yang, Pu, Zheng, and
  Gong}]{jiao2022incremental}
Jiao, B.; Guo, Y.; Yang, C.; Pu, J.; Zheng, Z.; and Gong, D. 2024.
\newblock Incremental Weighted Ensemble for Data Streams with Concept Drift.
\newblock \emph{IEEE Transactions on Artificial Intelligence}, 5(01): 92--103.

\bibitem[{Jiao et~al.(2023)Jiao, Guo, Yang, Pu, and Gong}]{jiao2023reduced}
Jiao, B.; Guo, Y.; Yang, S.; Pu, J.; and Gong, D. 2023.
\newblock Reduced-Space Multistream Classification Based on Multiobjective
  Evolutionary Optimization.
\newblock \emph{IEEE Transactions on Evolutionary Computation}, 27(4):
  764--777.

\bibitem[{Kim, Hwang, and Whang(2024)}]{kim2024quilt}
Kim, M.; Hwang, S.-H.; and Whang, S.~E. 2024.
\newblock Quilt: robust data segment selection against concept drifts.
\newblock In \emph{Proceedings of the AAAI Conference on Artificial
  Intelligence}, volume~38, 21249--21257.

\bibitem[{Korycki and Krawczyk(2021)}]{korycki2021concept}
Korycki, {\L}.; and Krawczyk, B. 2021.
\newblock Concept drift detection from multi-class imbalanced data streams.
\newblock In \emph{2021 IEEE 37th International Conference on Data Engineering
  (ICDE)}, 1068--1079. IEEE.

\bibitem[{Lei et~al.(2024)Lei, Chen, Wang, Jiang, and Zou}]{lei2024adapted}
Lei, T.; Chen, S.; Wang, B.; Jiang, Z.; and Zou, N. 2024.
\newblock Adapted-moe: Mixture of experts with test-time adaption for anomaly
  detection.
\newblock \emph{arXiv preprint arXiv:2409.05611}.

\bibitem[{Li et~al.()Li, Lin, Duan, Liang, and Shroff}]{litheory}
Li, H.; Lin, S.; Duan, L.; Liang, Y.; and Shroff, N. ????
\newblock Theory on Mixture-of-Experts in Continual Learning.
\newblock In \emph{The Thirteenth International Conference on Learning
  Representations}.

\bibitem[{Li et~al.(2022)Li, Yang, Liu, Xia, and Bian}]{li2022ddg}
Li, W.; Yang, X.; Liu, W.; Xia, Y.; and Bian, J. 2022.
\newblock Ddg-da: Data distribution generation for predictable concept drift
  adaptation.
\newblock In \emph{Proceedings of the AAAI Conference on Artificial
  Intelligence}, volume~36, 4092--4100.

\bibitem[{Liu, Lu, and Zhang(2020)}]{liu2020diverse}
Liu, A.; Lu, J.; and Zhang, G. 2020.
\newblock Diverse instance-weighting ensemble based on region drift
  disagreement for concept drift adaptation.
\newblock \emph{IEEE transactions on neural networks and learning systems},
  32(1): 293--307.

\bibitem[{Lu et~al.(2018)Lu, Liu, Dong, Gu, Gama, and Zhang}]{lu2018learning}
Lu, J.; Liu, A.; Dong, F.; Gu, F.; Gama, J.; and Zhang, G. 2018.
\newblock Learning under concept drift: A review.
\newblock \emph{IEEE transactions on knowledge and data engineering}, 31(12):
  2346--2363.

\bibitem[{Lu et~al.(2025)Lu, Lu, Liu, and Zhang}]{lu2025early}
Lu, P.; Lu, J.; Liu, A.; and Zhang, G. 2025.
\newblock Early Concept Drift Detection via Prediction Uncertainty.
\newblock In \emph{Proceedings of the AAAI Conference on Artificial
  Intelligence}, volume~39, 19124--19132.

\bibitem[{Ma et~al.(2024)Ma, Lu, Fang, Liu, and Zhang}]{ma2024multiview}
Ma, G.; Lu, J.; Fang, Z.; Liu, F.; and Zhang, G. 2024.
\newblock Multiview classification through learning from interval-valued data.
\newblock \emph{IEEE Transactions on Neural Networks and Learning Systems}.

\bibitem[{Marcu and Bouvry(2024)}]{marcu2024big}
Marcu, O.-C.; and Bouvry, P. 2024.
\newblock \emph{Big data stream processing}.
\newblock Ph.D. thesis, University of Luxembourg.

\bibitem[{Maurer, Pontil, and Romera-Paredes(2016)}]{maurer2016benefit}
Maurer, A.; Pontil, M.; and Romera-Paredes, B. 2016.
\newblock The benefit of multitask representation learning.
\newblock \emph{Journal of Machine Learning Research}, 17(81): 1--32.

\bibitem[{Montiel et~al.(2021)Montiel, Halford, Mastelini, Bolmier, Sourty,
  Vaysse, Zouitine, Gomes, Read, Abdessalem et~al.}]{montiel2021river}
Montiel, J.; Halford, M.; Mastelini, S.~M.; Bolmier, G.; Sourty, R.; Vaysse,
  R.; Zouitine, A.; Gomes, H.~M.; Read, J.; Abdessalem, T.; et~al. 2021.
\newblock River: machine learning for streaming data in python.
\newblock \emph{The Journal of Machine Learning Research}, 22(1): 4945--4952.

\bibitem[{Mourtada, Ga{\"\i}ffas, and Scornet(2021)}]{mourtada2021amf}
Mourtada, J.; Ga{\"\i}ffas, S.; and Scornet, E. 2021.
\newblock {AMF}: Aggregated Mondrian forests for online learning.
\newblock \emph{Journal of the Royal Statistical Society Series B: Statistical
  Methodology}, 83(3): 505--533.

\bibitem[{Mu and Lin(2025)}]{mu2025comprehensive}
Mu, S.; and Lin, S. 2025.
\newblock A comprehensive survey of mixture-of-experts: Algorithms, theory, and
  applications.
\newblock \emph{arXiv preprint arXiv:2503.07137}.

\bibitem[{Panchal et~al.(2023)Panchal, Choudhary, Mitra, Mukherjee, Sarkhel,
  Mitra, and Guan}]{panchal2023flash}
Panchal, K.; Choudhary, S.; Mitra, S.; Mukherjee, K.; Sarkhel, S.; Mitra, S.;
  and Guan, H. 2023.
\newblock Flash: Concept drift adaptation in federated learning.
\newblock In \emph{International Conference on Machine Learning}, 26931--26962.
  PMLR.

\bibitem[{Qin et~al.(2020)Qin, Cheng, Zhao, Chen, Metzler, and
  Qin}]{qin2020multitask}
Qin, Z.; Cheng, Y.; Zhao, Z.; Chen, Z.; Metzler, D.; and Qin, J. 2020.
\newblock Multitask mixture of sequential experts for user activity streams.
\newblock In \emph{Proceedings of the 26th ACM SIGKDD international conference
  on knowledge discovery \& data mining}, 3083--3091.

\bibitem[{Read and Zliobaite(2025)}]{read2025supervised}
Read, J.; and Zliobaite, I. 2025.
\newblock Supervised Learning from Data Streams: An Overview and Update.
\newblock \emph{ACM Computing Surveys}.

\bibitem[{Sarkar et~al.(2023)Sarkar, Liang, Fan, Wang, and
  Hao}]{sarkar2023edge}
Sarkar, R.; Liang, H.; Fan, Z.; Wang, Z.; and Hao, C. 2023.
\newblock Edge-moe: Memory-efficient multi-task vision transformer architecture
  with task-level sparsity via mixture-of-experts.
\newblock In \emph{2023 IEEE/ACM International Conference on Computer Aided
  Design (ICCAD)}, 01--09. IEEE.

\bibitem[{Smola et~al.(2007)Smola, Gretton, Song, and
  Sch{\"o}lkopf}]{smola2007hilbert}
Smola, A.; Gretton, A.; Song, L.; and Sch{\"o}lkopf, B. 2007.
\newblock A Hilbert space embedding for distributions.
\newblock In \emph{International conference on algorithmic learning theory},
  13--31. Springer.

\bibitem[{Song et~al.(2021)Song, Lu, Liu, Lu, and Zhang}]{song2021segment}
Song, Y.; Lu, J.; Liu, A.; Lu, H.; and Zhang, G. 2021.
\newblock A segment-based drift adaptation method for data streams.
\newblock \emph{IEEE transactions on neural networks and learning systems},
  33(9): 4876--4889.

\bibitem[{Street and Kim(2001)}]{Street2001}
Street, W.~N.; and Kim, Y. 2001.
\newblock A streaming ensemble algorithm ({SEA}) for large-scale
  classification.
\newblock In \emph{Proceedings of the 7th ACM SIGKDD International Conference
  on Knowledge Discovery and Data Mining}, 377--382. New York, USA: ACM.

\bibitem[{Tran, Pham et~al.(2025)}]{tran2025revisiting}
Tran, V.-T.; Pham, Q.-V.; et~al. 2025.
\newblock Revisiting Sparse Mixture of Experts for Resource-adaptive Federated
  Fine-tuning Foundation Models.
\newblock In \emph{ICLR 2025 Workshop on Modularity for Collaborative,
  Decentralized, and Continual Deep Learning}.

\bibitem[{Vaswani et~al.(2017)Vaswani, Shazeer, Parmar, Uszkoreit, Jones,
  Gomez, Kaiser, and Polosukhin}]{vaswani2017attention}
Vaswani, A.; Shazeer, N.; Parmar, N.; Uszkoreit, J.; Jones, L.; Gomez, A.~N.;
  Kaiser, {\L}.; and Polosukhin, I. 2017.
\newblock Attention is all you need.
\newblock \emph{Advances in neural information processing systems}, 30.

\bibitem[{Vyas et~al.(2014)Vyas, Kannao, Bhargava, and
  Guha}]{vyas2014commercial}
Vyas, A.; Kannao, R.; Bhargava, V.; and Guha, P. 2014.
\newblock Commercial block detection in broadcast news videos.
\newblock In \emph{Proceedings of the 2014 Indian Conference on Computer Vision
  Graphics and Image Processing}, 1--7.

\bibitem[{Wan, Liang, and Yoon(2024)}]{wan2024online}
Wan, K.; Liang, Y.; and Yoon, S. 2024.
\newblock Online drift detection with maximum concept discrepancy.
\newblock In \emph{Proceedings of the 30th ACM SIGKDD Conference on Knowledge
  Discovery and Data Mining}, 2924--2935.

\bibitem[{Wang et~al.(2024)Wang, Lu, Liu, and Zhang}]{wang2024adaptive}
Wang, K.; Lu, J.; Liu, A.; and Zhang, G. 2024.
\newblock An Adaptive Stacking Method for Multiple Data Streams Learning under
  Concept Drift.
\newblock In \emph{Intelligent Management of Data and Information in Decision
  Making: Proceedings of the 16th FLINS Conference on Computational
  Intelligence in Decision and Control \& the 19th ISKE Conference on
  Intelligence Systems and Knowledge Engineering (FLINS-ISKE 2024)}, 267--274.
  World Scientific.

\bibitem[{Wang et~al.(2021)Wang, Lu, Liu, Zhang, and Xiong}]{wang2021evolving}
Wang, K.; Lu, J.; Liu, A.; Zhang, G.; and Xiong, L. 2021.
\newblock Evolving gradient boost: A pruning scheme based on loss improvement
  ratio for learning under concept drift.
\newblock \emph{IEEE Transactions on Cybernetics}, 53(4): 2110--2123.

\bibitem[{Wen et~al.(2023)Wen, Chen, Sun, Zhang, Wang, Jin, Tan
  et~al.}]{wen2023onenet}
Wen, Q.; Chen, W.; Sun, L.; Zhang, Z.; Wang, L.; Jin, R.; Tan, T.; et~al. 2023.
\newblock Onenet: Enhancing time series forecasting models under concept drift
  by online ensembling.
\newblock \emph{Advances in Neural Information Processing Systems}, 36:
  69949--69980.

\bibitem[{Xiang et~al.(2023)Xiang, Zi, Cong, and Wang}]{xiang2023concept}
Xiang, Q.; Zi, L.; Cong, X.; and Wang, Y. 2023.
\newblock Concept drift adaptation methods under the deep learning framework: A
  literature review.
\newblock \emph{Applied Sciences}, 13(11): 6515.

\bibitem[{Xu, Chen, and Wang(2025{\natexlab{a}})}]{xu2025coral}
Xu, K.; Chen, L.; and Wang, S. 2025{\natexlab{a}}.
\newblock Coral: Concept drift representation learning for co-evolving
  time-series.
\newblock \emph{arXiv preprint arXiv:2501.01480}.

\bibitem[{Xu, Chen, and Wang(2025{\natexlab{b}})}]{xu2025drift2matrix}
Xu, K.; Chen, L.; and Wang, S. 2025{\natexlab{b}}.
\newblock Drift2matrix: Kernel-induced self representation for concept drift
  adaptation in co-evolving time series.

\bibitem[{Yang, Lu, and Yu(2025)}]{yang2025adapting}
Yang, X.; Lu, J.; and Yu, E. 2025.
\newblock Adapting Multi-modal Large Language Model to Concept Drift From
  Pre-training Onwards.
\newblock In \emph{The Thirteenth International Conference on Learning
  Representations}.

\bibitem[{Yu et~al.(2025)Yu, Lu, Yang, Zhang, and Fang}]{yu2025learning}
Yu, E.; Lu, J.; Yang, X.; Zhang, G.; and Fang, Z. 2025.
\newblock Learning Robust Spectral Dynamics for Temporal Domain Generalization.
\newblock \emph{arXiv preprint arXiv:2505.12585}.

\bibitem[{Yu et~al.(2024)Yu, Lu, Zhang, and Zhang}]{yu2024online}
Yu, E.; Lu, J.; Zhang, B.; and Zhang, G. 2024.
\newblock Online boosting adaptive learning under concept drift for multistream
  classification.
\newblock In \emph{Proceedings of the AAAI Conference on Artificial
  Intelligence}, volume~38, 16522--16530.

\bibitem[{Yu, Lu, and Zhang(2024)}]{yu2024fuzzy}
Yu, E.; Lu, J.; and Zhang, G. 2024.
\newblock Fuzzy Shared Representation Learning for Multistream Classification.
\newblock \emph{IEEE Transactions on Fuzzy Systems}, 32(10): 5625--5637.

\bibitem[{Zhang et~al.(2025)Zhang, Yu, Shao, and Sun}]{zhang2025multimodal}
Zhang, T.; Yu, E.; Shao, Y.; and Sun, J. 2025.
\newblock Multimodal Inverse Attention Network with Intrinsic Discriminant
  Feature Exploitation for Fake News Detection.
\newblock \emph{arXiv preprint arXiv:2502.01699}.

\end{thebibliography}
\clearpage
\newpage

\appendix
\section*{Appendix}
\section{A. Theoretical Analysis}
In this section, we provide a formal theoretical analysis to ground the design of our CAMEL framework. We leverage the well-established theory of generalization bounds for multi-task learning (MTL)~\cite{maurer2016benefit} to rationalize the core components of our architecture. The analysis demonstrates how CAMEL's design inherently balances knowledge transfer and task-specific specialization to achieve robust performance in the HML setting. To better understand our theoretical analysis, we give some definitions:
\begin{definition}[Window-wise Risk]
For a hypothesis $h \in \mathcal{F}$ and a stream $S_i$ in window $W_t$, the true risk is the expected loss $\mathcal{R}_i(h) = \mathbb{E}_{(x,y) \sim P^{(i)}}[\ell(h(x), y)]$. The empirical risk on the data window $W_{i,t}$ is $\hat{\mathcal{R}}_i(h) = \frac{1}{|W_{i,t}|} \sum_{(x_k, y_k) \in W_{i,t}} \ell(h(x_k), y_k)$.
\end{definition}

\begin{definition}[Inter-Stream Dissimilarity]
The dissimilarity between streams in $\mathcal{S} = \{S_i\}_{i=1}^n$ is defined as the maximum deviation between any single stream's risk and the average risk across all streams, measured over the entire hypothesis space $\mathcal{F}$.
\begin{equation}
    C(\{\mathcal{S}_j\}) = \sup_{h \in \mathcal{F}} \max_{i} \left| \mathcal{R}_i(h) - \frac{1}{n}\sum_{j=1}^n \mathcal{R}_j(h) \right|
\end{equation}
This metric quantifies the heterogeneity of the learning tasks. A small $C(\{\mathcal{S}_j\})$ implies that the streams represent closely related tasks, whereas a large value indicates significant task divergence.
\end{definition}


\subsection{Generalization Bound}

We restate the main generalization bound for CAMEL, which connects the performance on a single stream to the average performance across all streams and their dissimilarity.
\begin{theorem}[Restatement of  Theorem~\ref{thm:gene_bound}]
\label{thm:main_bound}
Let $h$ be a hypothesis learned by CAMEL on data from window $W_t$. Assume all windows $|W_{i,t}|$ are equal to $|W_t|$. Then, for any stream $S_i$, with probability at least $1-\delta$ over the random draw of the training samples, the following bound holds for all $h \in \mathcal{F}$:
\begin{equation}
\begin{aligned}
       \mathcal{R}_i(h) & \leq  \hat{\mathcal{R}}_{\text{avg}}(h) + C(\{\mathcal{S}_i\}_{i=1}^n) \\
       & + \sqrt{\frac{d(\log(2n|W_t|/d)+1) - \log(\delta/4)}{2n|W_t|}} 
\end{aligned}
\end{equation}
where $\hat{\mathcal{R}}_{\text{avg}}(h) = \frac{1}{n} \sum_{j=1}^n \hat{\mathcal{R}}_j(h)$ is the average empirical risk. And $C(\{\mathcal{S}_i\}_{i=1}^n)$ quantifies the inter-stream dissimilarity. 
\end{theorem}

\begin{proof}
The proof follows the standard argument for multi-task learning bounds \cite{maurer2016benefit}. We begin by decomposing the risk of stream $S_i$:
\begin{equation}
\begin{aligned}
    \mathcal{R}_i(h) &= (\mathcal{R}_i(h) - \mathcal{R}_{\text{avg}}(h)) \\ 
    &+ (\mathcal{R}_{\text{avg}}(h) - \hat{\mathcal{R}}_{\text{avg}}(h)) + \hat{\mathcal{R}}_{\text{avg}}(h)
\end{aligned}
\end{equation}
where $\mathcal{R}_{\text{avg}}(h) = \frac{1}{n} \sum_{j=1}^n \mathcal{R}_j(h)$ is the average true risk.

We bound the first two terms on the right-hand side separately:

\noindent
\textbf{1) Bounding the first term (Dissimilarity):} By Definition 2, the first term is bounded by the inter-stream dissimilarity:
\begin{equation}
\small
    \mathcal{R}_i(h) - \mathcal{R}_{\text{avg}}(h) \leq \sup_{h' \in \mathcal{F}} (\mathcal{R}_i(h') - \mathcal{R}_{\text{avg}}(h')) \leq C(\{\mathcal{S}_j\})
\end{equation}

\noindent
\textbf{2) Bounding the second term (Generalization Error):} The second term is the generalization error of the average risk. We can apply a standard VC-dimension bound to the average hypothesis over a total of $n|W_t|$ samples drawn from the mixture distribution $P_{\text{avg}} = \frac{1}{n}\sum_j P^{(j)}$. With probability at least $1-\delta/2$, for all $h \in \mathcal{F}$:
\begin{equation}
\small
\mathcal{R}_{\text{avg}}(h) - \hat{\mathcal{R}}_{\text{avg}}(h) \leq \sqrt{\frac{d(\log(2n|W_t|/d)+1) - \log(\delta/4)}{2n|W_t|}}
\end{equation}

\noindent
\textbf{Combining the bounds:} Substituting the bounds for the first two terms back into the decomposition, and applying a union bound for the probabilities, we arrive at the final result stated in the theorem.
\end{proof}

\subsection{Theoretical Rationale for CAMEL's Architecture}

Theorem~\ref{thm:gene_bound} formally establishes the core trade-off in heterogeneous multistream learning: balancing the benefit of joint training (lower average empirical risk $\hat{\mathcal{R}}_{\text{avg}}$) against the penalty of task divergence $C(\{\mathcal{S}_j\})$. The architecture of CAMEL is a direct embodiment of this principle, with each component designed to optimize this trade-off.

\paragraph{Assistance Expert (\texttt{AE}) as a Dissimilarity Minimizer.}
The dissimilarity term $C(\{\mathcal{S}_j\})$, while defined on unobservable true risks, can be minimized through a proxy strategy: learning aligned feature representations. The \texttt{AE} serves precisely this function. Its attention mechanism identifies and fuses relevant cross-stream information, effectively creating a shared representation subspace that reduces task divergence and thus tightens the generalization bound.

\paragraph{Routing Network (\texttt{RN}) as an Adaptive Trade-off Controller.}
The \texttt{RN} operationalizes the trade-off between collaboration and specialization. By learning to route each input, it dynamically determines the optimal degree of knowledge transfer for that specific instance. Through end-to-end optimization, the \texttt{RN} is incentivized to favor the AE when collaboration is beneficial and to rely on private experts otherwise. This behavior constitutes a learned, adaptive solution to balancing the terms in the generalization bound.

\paragraph{Autonomous Expert Tuner (\texttt{AET}) for Dynamic Stability.}
The \texttt{AET} extends this framework to the non-stationary streaming setting. As concept drift alters the underlying data distributions ($P^{(i)}$) and dissimilarity ($C(\{\mathcal{S}_j\})$), the \texttt{AET} maintains performance by adaptively managing the hypothesis space $\mathcal{F}$ itself. It ensures controlled capacity growth by instantiating new experts for new concepts while freezing past ones to prevent catastrophic forgetting. This modular plasticity ensures the generalization bound remains meaningful and the model stays robust across evolving data streams.

\section{B. Experiment Settings}
\subsection{B.1. Datasets}
\begin{table*}[t]
\centering
\begin{tabular}{ccccccl}
\toprule
\textbf{Scenarios} & \textbf{\#Datasets} & \textbf{\#Sample} & \textbf{\#Feature} & \textbf{\#Class} & \textbf{\#Drift type} \\
\midrule
\multirow{3}{*}{Set 1: Tree (Homo.)} 
    &  $\mathcal{S}_1$     & 5000  & 20  & 2  & Sudden/gradual\\
    &  $\mathcal{S}_2$     & 5000  & 20  & 2  & Sudden/gradual\\
    &  $\mathcal{S}_3$     & 5000  & 20  & 2  & Sudden/gradual\\
\midrule
\multirow{3}{*}{Set 2: Hyperplane (Homo.)} 
    &   $\mathcal{S}_1$   & 30000   & 4  & 2  & Incremental\\
    &   $\mathcal{S}_2$   & 30000   & 4  & 2  & Incremental\\
    &   $\mathcal{S}_3$   & 30000   & 4  & 2  & Incremental\\
\midrule
\multirow{3}{*}{Set 3 (Hete.)} 
   & SEAa & 10,000 & 3 & 2  & Sudden \\
 & RTG & 10,000 & 10 & 2  & No \\
& RBF & 10,000 & 10 & 2  & Incremental \\
\midrule
\multirow{3}{*}{Set 4 (Hete.)} 
    & LED       & 100,000 & 7  & 24 & Noise       \\
    & LEDDrift  & 100,000 & 24 & 24 & Unknown     \\
    & Waveform  & 100,000 & 39 & 3  & Noise       \\
\midrule
\multirow{3}{*}{Set 5: TV News (Homo.)} 
    & CNNIBN   & 30,000 & 124  & 2 & Unknown \\
    & BBC      & 30,000 & 124  & 2 & Unknown \\
    & TIMENEWS & 30,000 & 124  & 2 & Unknown \\
\midrule
\multirow{3}{*}{Set 6: Weather (Homo.)} 
    & $\mathcal{S}_1$ & 45000 & 8 & 2 & Unknown \\
    & $\mathcal{S}_2$ & 45000 & 8 & 2 & Unknown \\
    & $\mathcal{S}_3$ & 45000 & 8 & 2 & Unknown \\
\midrule
\multirow{3}{*}{Scenario 7: Credit Card (Hete.)} 
    & $\mathcal{S}_1$ & 30,000 & 5 & 2  &  Unknown \\
    & $\mathcal{S}_2$ & 30,000 & 6 & 2  &  Unknown \\
    & $\mathcal{S}_3$ & 30,000 & 12 & 2  &  Unknown \\
\midrule
\multirow{3}{*}{Scenario 8: Covertype (Hete.)} 
    & $\mathcal{S}_1$ & 581,012 & 10 & 7 & Unknown \\
    & $\mathcal{S}_2$ & 581,012 & 4  & 7 & Unknown \\
    & $\mathcal{S}_3$ & 581,012 & 40 & 7 & Unknown \\
\bottomrule
\end{tabular}
\caption{Multiple Data Streams Scenarios: characteristics of all datasets.}
\label{tab:datasets}
\end{table*}
To comprehensively evaluate CAMEL under diverse HML settings, we constructed a benchmark of eight multistream scenarios, comprising four synthetic and four real-world sets, as detailed in Table~\ref{tab:datasets}. Each scenario consists of three concurrent data streams.

\textbf{1) Synthetic Scenarios:} We established four distinct synthetic scenarios to systematically test the framework's capabilities against controlled challenges. 
Two scenarios feature \textit{homogeneous} feature spaces: \textit{Set 1 (Tree)}~\cite{liu2020diverse} and \textit{Set 2 (Hyperplane)}~\cite{bifet2007learning}, for which the data generation process follows the methodology outlined in~\cite{yu2024online}.
To assess performance on feature heterogeneity, \textit{Set 3} is a \textit{heterogeneous} composite of three classic benchmarks: SEAa~\cite{Street2001}, RTG~\cite{Domingos2000}, and a stream generated by a radial basis function (RBF) generator~\cite{song2021segment}.
\textit{Set 4} further tests adaptability to varying data complexity and noise, comprising three well-known datasets from the River library~\cite{montiel2021river}: LED, LEDDrift, and Waveform.

\textbf{2) Real-World Scenarios:} To validate CAMEL's efficacy on practical tasks, we employ four real-world multistream benchmarks. 
For the \textbf{homogeneous} settings, we use the \textit{Set 5 TV News}\footnote{https://archive.ics.uci.edu/dataset/326/tv+news+channel\\+commercial+detection+dataset}~\cite{vyas2014commercial} and \textit{Set 6 Weather}~\cite{ditzler2012incremental} datasets. Following the procedure in~\cite{yu2024online}, we partition the TV News data into three streams (CNNIBN, BBC, TIMES) and select three representative streams from the Weather dataset.
For the \textbf{heterogeneous} settings, we create two scenarios from widely-used datasets. In \textit{Set 7 (Credit Card)}\footnote{https://www.kaggle.com/datasets/samuelcortinhas/credit-card-classification-clean-data/data}, we split the original dataset into three streams based on distinct user payment behaviors, resulting in different feature spaces. Similarly, for \textit{Set 8 (Covertype)}\footnote{https://archive.ics.uci.edu/dataset/31/covertype}, we partition the data into three streams according to different feature categories, creating another challenging heterogeneous scenario.

\subsection{B.2. Baselines}
\label{app:baselines}
To validate the performance of our proposed \mymethod~framework, we conduct a comprehensive comparison against two categories of state-of-the-art methods: established single-stream online learning algorithms and contemporary multistream classification frameworks. For all baselines, we adhere to the parameter settings recommended in their original publications, ensuring a fair and rigorous evaluation. 


\subsubsection{Single-Stream Baselines.}
These methods represent the standard approach where each data stream is learned independently without any knowledge fusion. We apply each of these algorithms to every stream in our scenarios and report the average performance.
\begin{itemize}
    \item Streaming Random Patches (SRP)~\cite{gomes2019streaming}: An ensemble method for data streams that learns from random patches of features. We utilize its random subspace mode as a robust baseline.
    \item Aggregated Mondrian Forest (AMF)~\cite{mourtada2021amf}: A highly efficient online random forest algorithm based on Mondrian processes, well-suited for evolving data.
    \item Incremental Weighted Ensemble (IWE)~\cite{jiao2022incremental}: A chunk-based ensemble method that adapts to concept drifts by dynamically weighting its base learners.
\end{itemize}

\subsubsection{Multistream Classification Baselines.}
This category includes recent methods specifically designed for multistream learning, although they primarily target homogeneous data settings. To adapt them to our n-stream scenarios, we follow a common evaluation protocol: for any given target stream $\mathcal{S}_i$, the remaining $n-1$ streams serve as the source streams.
\begin{itemize}
    \item MCMO~\cite{jiao2023reduced}: A multistream classification framework based on multi-objective evolutionary optimization.
    \item OBAL~\cite{yu2024online}: An online boosting adaptive learning algorithm that dynamically weights source streams based on drift-awareness.
    \item BFSRL~\cite{yu2024fuzzy}: A method that learns fuzzy shared representations to handle correlations across multiple streams.
\end{itemize}
It is important to note that these methods are not inherently designed for the full spectrum of HML challenges, particularly heterogeneous feature and label spaces. 


\subsection{B.3. Implementation Details}
\label{appendix:imple}

\begin{figure*}[thbp]
    \centering
    \begin{subfigure}[t]{\columnwidth}
        \centering
      \includegraphics[width=\linewidth]{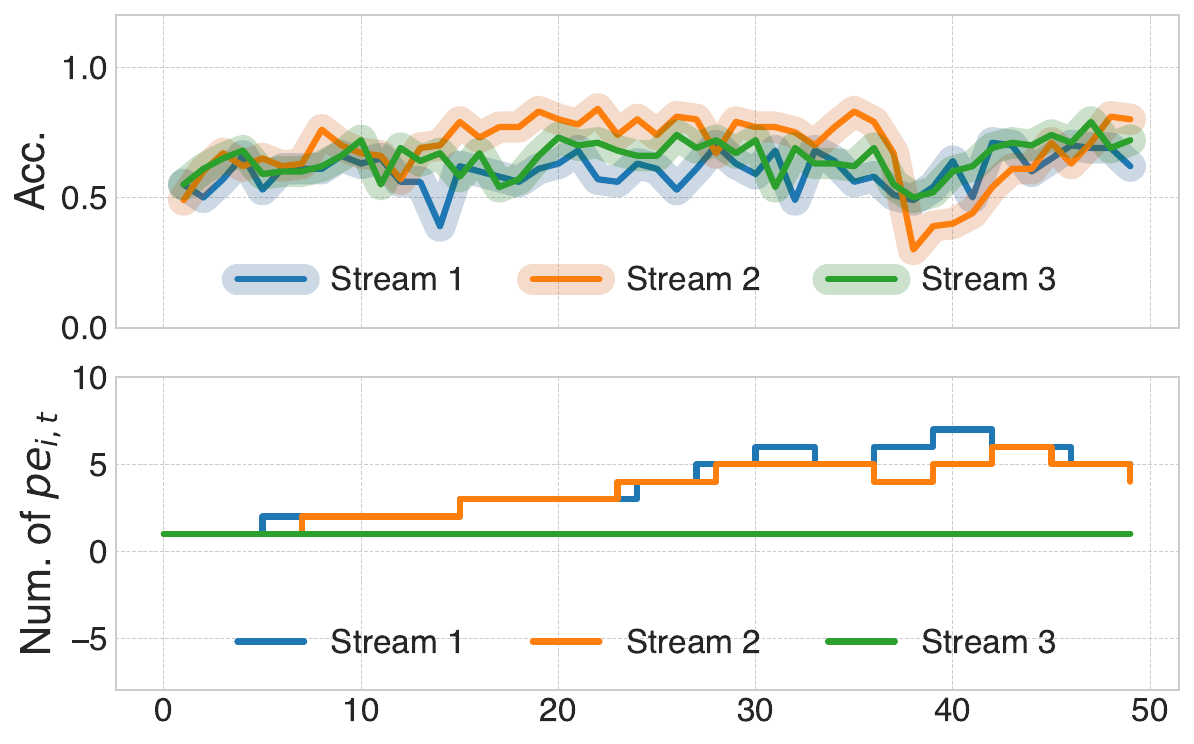}
        \caption{Set 1 Tree}
        \label{fig:online_tree}
    \end{subfigure}
    \begin{subfigure}[t]{\columnwidth}
        \centering
        \includegraphics[width=\linewidth]{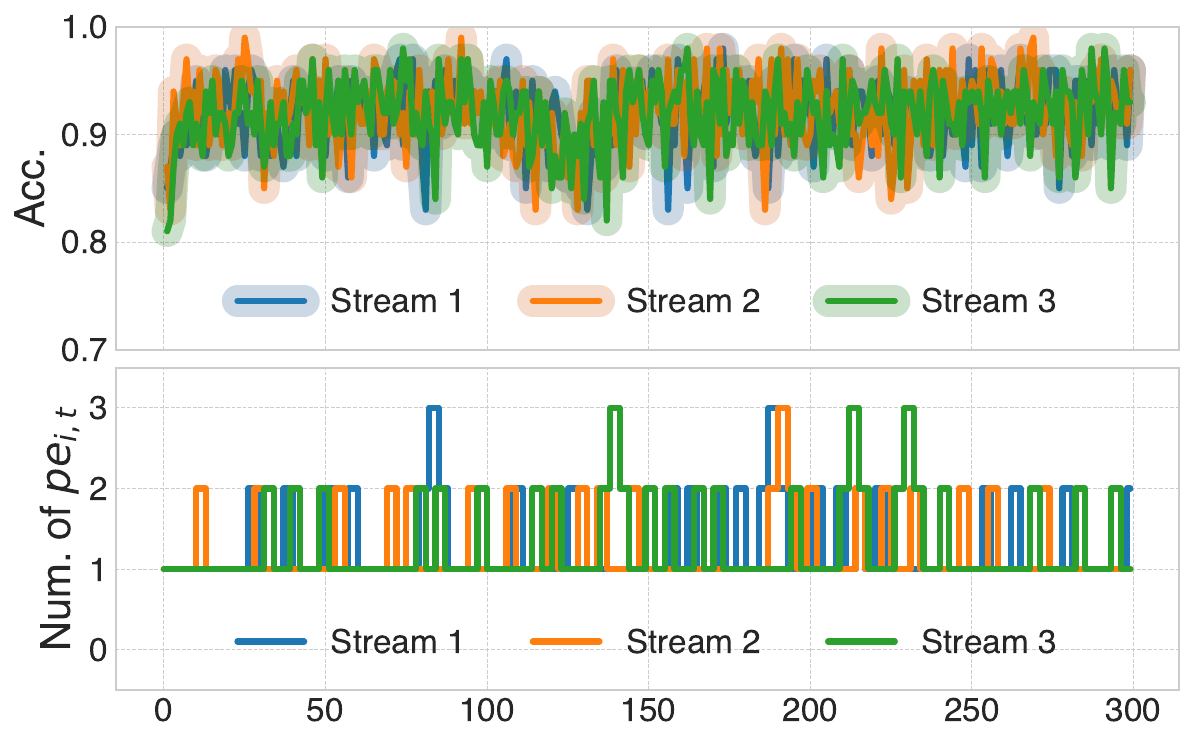}
        \caption{Set 2 Hyperplane}
        \label{fig:online_hyper}
    \end{subfigure}
    \begin{subfigure}[t]{\columnwidth}
        \centering
        \includegraphics[width=\linewidth]{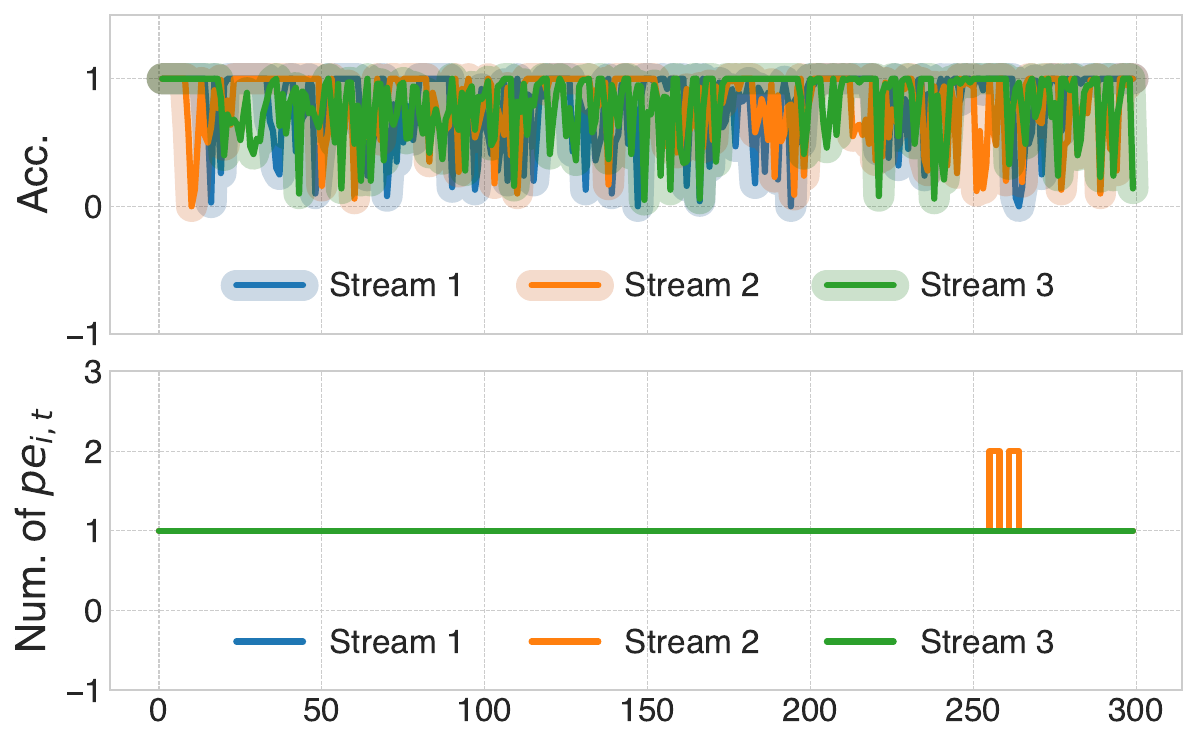}
        \caption{Set 5 TV News}
        \label{fig:online_tv}
    \end{subfigure}
    \begin{subfigure}[t]{\columnwidth}
        \centering
        \includegraphics[width=\linewidth]{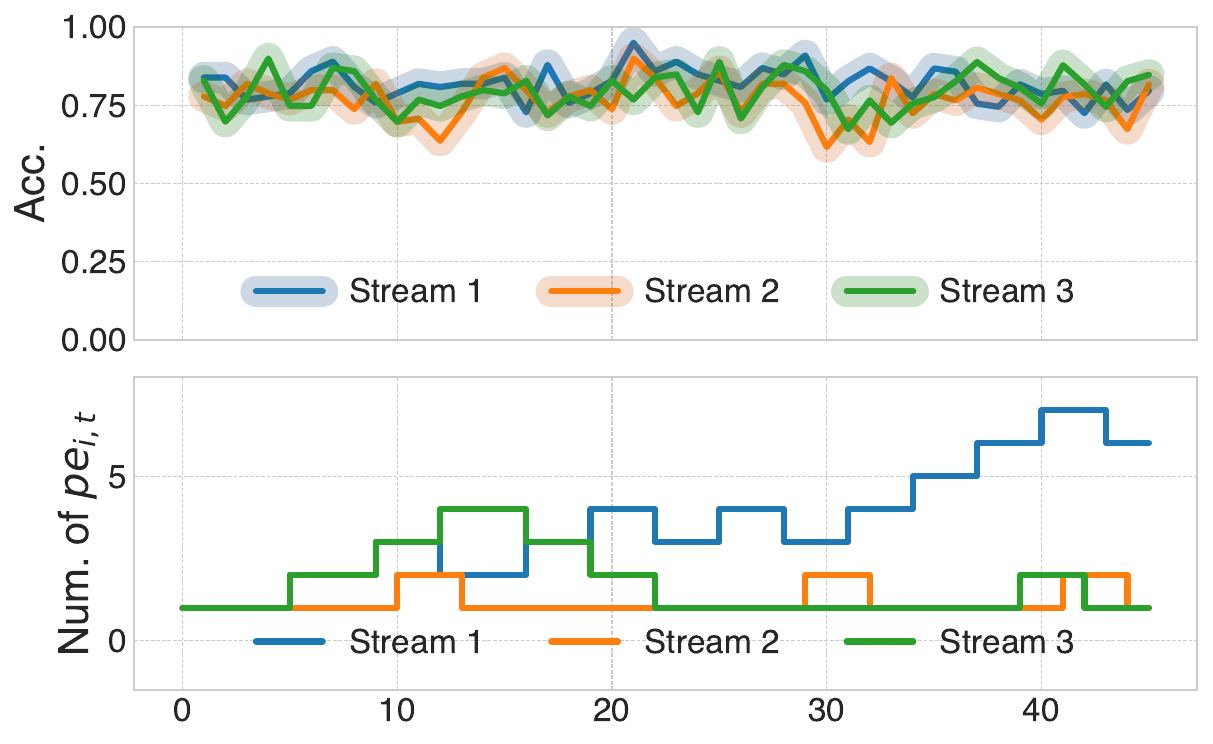}
        \caption{Set 6 Weather}
        \label{fig:online_weather}
    \end{subfigure}
    \begin{subfigure}[t]{\columnwidth}
        \centering
        \includegraphics[width=\linewidth]{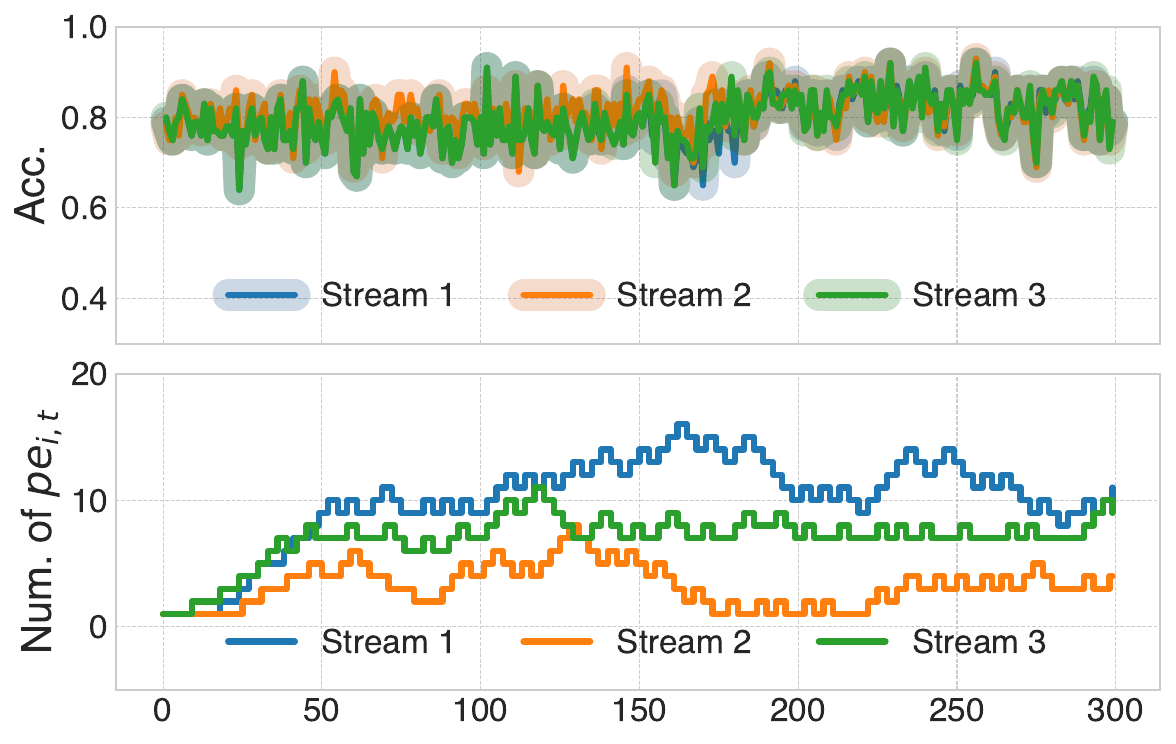}
        \caption{Set 7 Credit Card}
        \label{fig:online_credit}
    \end{subfigure}
    \begin{subfigure}[t]{\columnwidth}
        \centering
        \includegraphics[width=\linewidth]{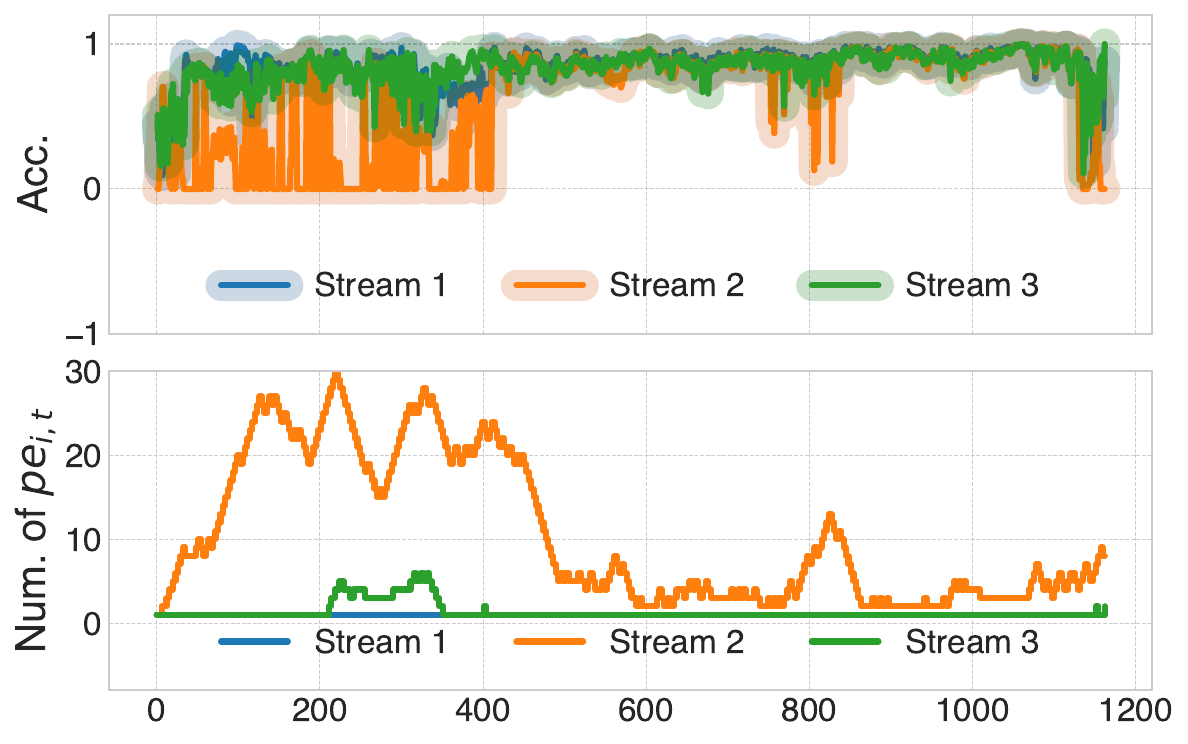}
        \caption{Set 8 Covertype}
        \label{fig:online_cover}
    \end{subfigure}
    \caption{Online accuracy and the corresponding number of private experts over time.}
    \label{fig:online_app}
\end{figure*}

We implement our CAMEL framework in PyTorch. The specific architectural configurations and training hyperparameters are detailed below.

\subsubsection{Network Architecture.}

Our architecture comprises the following key modules per stream $\mathcal{S}_i$:
The Feature Extractor ($\texttt{FE}_i$: Linear ($D_i \to 50$) $\to$ ReLU $\to$ Linear ($50 \to 50$) $\to$ ReLU $\to$ Linear ($50 \to D_h$)) maps stream-specific inputs ($D_i$) to a shared latent space ($D_h$). 
Private Experts ($\texttt{PE}_i(t)$) and the Assistance Expert ($\texttt{AE}_i$) both process features into refined representations of dimension $D_f$: 
$\texttt{PE}_i(t)$ uses Linear ($D_h \to 50$) $\to$ ReLU $\to$ Linear ($50 \to 50$) $\to$ ReLU $\to$ Linear ($50 \to D_f$), 
while $\texttt{AE}_i$ employs multi-head attention (2 heads) on $\boldsymbol{h}_{i,t}$, then processes the concatenated $[\boldsymbol{h}_{i,t}; \boldsymbol{c}_{i,t}]$ (dim $2D_h$) via Linear($2D_h \to 50$) $\to$ ReLU $\to$ Linear ($50 \to D_f$). 
The Routing Network ($\texttt{RN}_i$: Linear ($D_h \to 50$) $\to$ ReLU $\to$ Linear ($50 \to 50$) $\to$ ReLU $\to$ Linear ($50 \to K(t)$) $\to$ Softmax) generates expert weights. A Classification Head ($\texttt{CH}_i$: Linear($D_f \to C_i$)) produces logits for the stream's label space ($C_i$).

\begin{figure*}[thbp]
    \centering
    \begin{subfigure}[t]{0.32\linewidth}
        \centering
      \includegraphics[width=\linewidth]{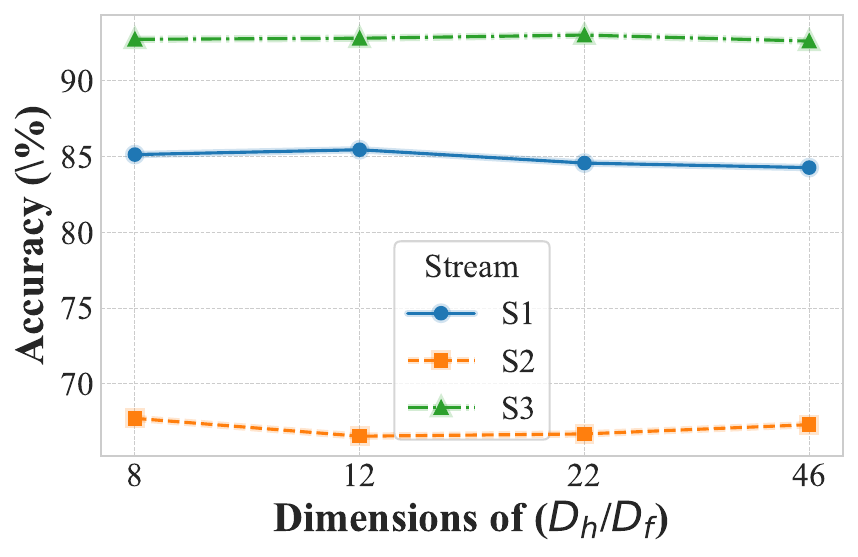}
        \caption{$D_h/D_f$.}
        \label{fig:dimension_SEA}
    \end{subfigure}
    \begin{subfigure}[t]{0.32\linewidth}
        \centering
        \includegraphics[width=\linewidth]{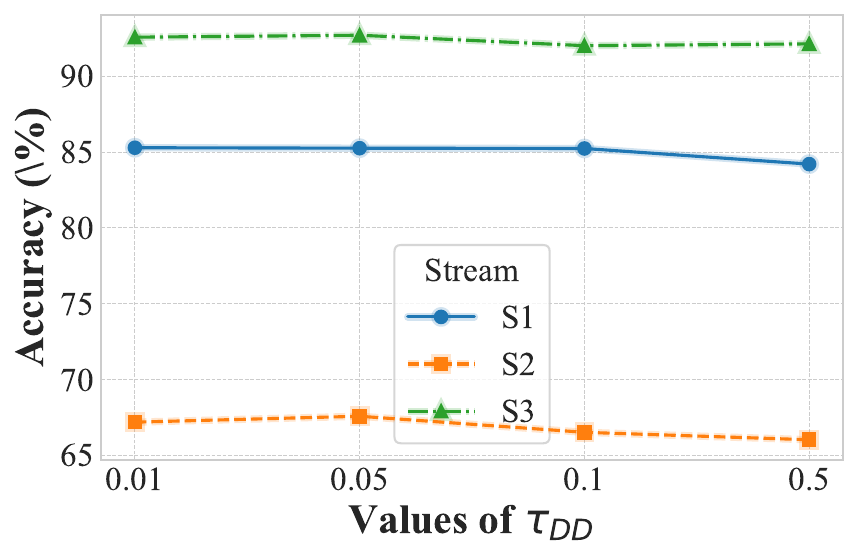}
        \caption{$\tau_{DD}$.}
        \label{fig:DD_SEA}
    \end{subfigure}
    \begin{subfigure}[t]{0.32\linewidth}
        \centering
        \includegraphics[width=\linewidth]{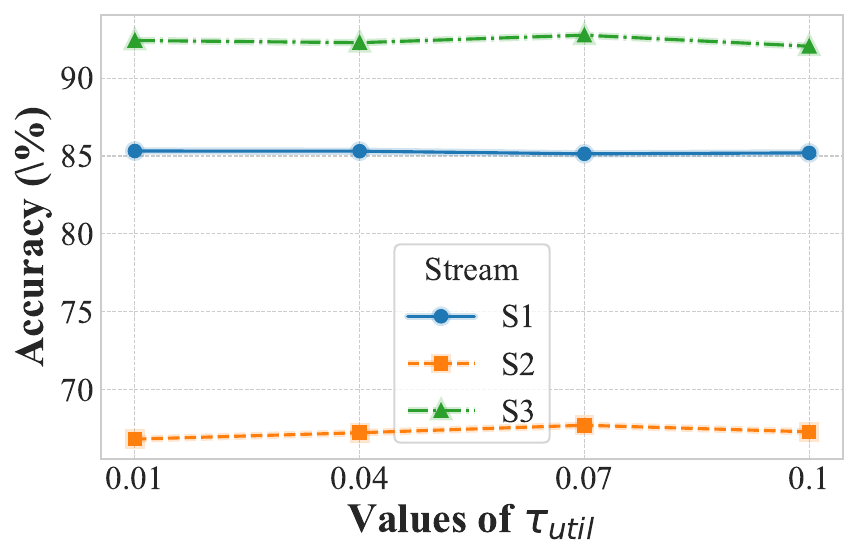}
        \caption{$\tau_{util}$.}
        \label{fig:ultis}
    \end{subfigure}
    \caption{Parameter analysis on Set 3: SEAa, RTG, RBF.}
    \label{fig:para_SEA}
\end{figure*}

\subsubsection{Online Learning and Adaptation Parameters.}
We configure the online learning process with the following hyperparameters:
The window-based prequential protocol uses a non-overlapping window size $|W|=500$ for Covertype and $|W|=100$ for other datasets. The model is initialized with 100 epochs of training on $W_0$; subsequent windows ($W_t, t \geq 1$) train for 30 epochs using the Adam optimizer ($\eta=1 \times 10^{-4}$).
Each stream's MMD-based drift detector ($\texttt{DD}_i$) employs an RBF kernel ($\sigma=0.15$) with a reference window size $|W|/4$.
The Autonomous Expert Tuner ($\texttt{AET}_i$) enforces a 2-window cooldown between architectural changes. Pruning occurs if a private expert's long-term average utilization falls below $\tau_{util}$ (minimum 1 expert per stream). New expert instantiation is triggered by a drift signal from $\texttt{DD}_i$ coupled with a significant performance drop (tracked via a lookback window of 5 and drop factor 0.95).

We implemented the framework using the PyTorch library. All experimental evaluations were conducted on a server equipped with 187GB of memory and powered by an Intel(R) Xeon(R) Gold 6226R CPU @ 2.90GHz.


\subsection{C. Supplementary Experiments}
\textbf{C.1. Online Performance.}
To further validate the robustness and generality of our \mymethod~framework, we present the online performance and dynamic adaptation behavior on the remaining six experimental scenarios, as shown in Figure~\ref{fig:online_app}  Across most scenarios, results reinforce key findings: On synthetic \textit{Set 1 (Tree)} and \textit{Set 2 (Hyperplane)} (Figures~\ref{fig:online_tree} and~\ref{fig:online_hyper}), the Autonomous Expert Tuner ($\texttt{AET}$) actively manages private experts in response to frequent drifts, sustaining high accuracy. Similarly, on real-world \textit{Set 7 (Credit Card)} (Figure~\ref{fig:online_credit}), gradual expert growth reflects continuous adaptation to evolving payment patterns, correlating with accuracy gains. The behavior on \textit{Set 6 (Weather)} and \textit{Set 8 (Covertype)} also demonstrates capacity adjustments corresponding to the underlying data complexity.

A particularly insightful case is presented by the \textit{Set 5 (TV News)} dataset (Figure~\ref{fig:online_tv}). Despite high accuracy volatility across streams, each stream's $\texttt{AET}$ maintains stable expert counts. This reveals that performance volatility in real-world streams can stem from numerous factors beyond true concept drift, such as sampling noise and feature-label ambiguity. Relying solely on a performance-drop trigger would likely lead to excessive and spurious architectural changes. However, $\texttt{AET}$ requires both a significant performance drop and an MMD-based drift signal ($\texttt{DD}$) to instantiate new experts. This conjunctive condition filters noise, correctly identifying TV News volatility as non-distributional shift, thus preventing unnecessary adaptations and preserving stability. This validates the necessity of integrating performance-based and distributional signals for robust online adaptation.

\noindent
\textbf{C.2. Parameter Sensitivity.}
We analyze the sensitivity of \mymethod~to its three key hyperparameters: the feature space dimensions ($D_h, D_f$, where we set $D_h=D_f$), the drift detection threshold ($\tau_{DD}$), and the expert pruning threshold ($\tau_{util}$). Figure~\ref{fig:para_SEA} illustrates the results on the heterogeneous Set 3, which is representative of the general findings.

We determined feature dimensions ($D_h$, $D_f$) based on the streams' average input dimension $D_{avg}$, testing values in $\{ \frac{1}{3}D_{avg}, \frac{1}{2}D_{avg}, D_{avg}, 2D_{avg} \}$ (ensured even for 2-head attention). Thresholds were tested over $\tau_{DD} \in \{0.01, 0.05, 0.1, 0.5\}$ and $\tau_{util} \in \{0.01, 0.04, 0.07, 0.1\}$.
As shown in Figure~\ref{fig:para_SEA}, the performance remains remarkably stable across all tested values for these parameters. This demonstrates that \mymethod~is robust and not overly sensitive to the precise setting of these key hyperparameters, which reduces the burden of parameter tuning. Detailed parameter settings in our experiments are shown in Table~\ref{tab:parameters}. 

\begin{table}[htbp]
  \centering
    \begin{tabular}{cccc}
    \toprule
          & $D_h$/$D_f$ & $\tau_{DD}$ & $\tau_{util}$ \\
    \midrule
    Set 1 & 20    & 0.05   & 0.07 \\
    Set 2  & 4     & 0.05   & 0.07 \\
    Set 3  & 8     & 0.07   & 0.1 \\
    Set 4 & 30    & 0.1     & 0.1 \\
    Set 5 & 124   & 0.5     & 0.1 \\
    Set 6 & 8     & 0.1     & 0.05 \\
    Set 7 & 8     & 0.07   & 0.05 \\
    Set 8 & 30    & 0.08   & 0.07 \\
    \bottomrule
    \end{tabular}%
      \caption{Parameter settings on different HML settings.}
  \label{tab:parameters}%
\end{table}%

\end{document}